%% file: paper.tex
\documentclass{article}
\usepackage[preprint]{neurips_2026}

\usepackage{microtype}
\usepackage{graphicx}
\usepackage{subfig}
\usepackage{hyperref}

\usepackage{amsmath}
\usepackage{amssymb}
\usepackage{mathtools}
\usepackage{amsthm}
\usepackage{thmtools,thm-restate}

\theoremstyle{plain}
\newtheorem{theorem}{Theorem}[section]

\newtheorem{corollary}[theorem]{Corollary}
\theoremstyle{definition}

\theoremstyle{remark}

\usepackage[textsize=tiny]{todonotes}

\usepackage{latexsym}
\usepackage{wrapfig}
\usepackage{algorithm}
\usepackage{algorithmic}
\usepackage{amsmath,amssymb}
\usepackage{xspace}
\usepackage{booktabs}
\usepackage{tcolorbox}
\usepackage{pgfplotstable, pgfmath}
\usepackage[noenumitem,notheorems]{eda}
\usepackage{bbm}
\usepackage{colortbl}
\usepackage{prettyplots}

\pgfplotsset{
    colormap/viridis,
}
\graphicspath{ {./figs/} }
\pgfplotsset{compat=1.18}
\usetikzlibrary{calc} 
\usetikzlibrary{positioning}
\usetikzlibrary{backgrounds}
\usetikzlibrary{external}
\usetikzlibrary{intersections}
\usepgfplotslibrary{fillbetween}
\usepgfplotslibrary{colorbrewer}
\usetikzlibrary{pgfplots.colormaps} 
\usetikzlibrary{pgfplots.statistics}
\usetikzlibrary{pgfplots.groupplots}
\usepgfplotslibrary{ternary}
\usetikzlibrary{fit}
\usetikzlibrary{arrows}
\usepackage{tikzviolinplots}
\usepackage{pgfplotstable}
\usepackage{scontents}
\usepackage{array}
\usepackage{numprint}
\usepackage{ifthen}
\usepackage{framed}
\usepackage{enumitem} 
\usepackage{wrapfig}
\usepackage{listings}
\usepackage{rotating}
\usepackage{color}

\usepackage[utf8]{inputenc}
\usepackage[T1]{fontenc}
\usepackage{hyperref}
\usepackage{url}
\usepackage{booktabs}
\usepackage{amsfonts}
\usepackage{nicefrac}
\usepackage{microtype}
\usepackage{xcolor}

\include{macro/defines}

\newcommand{\ourmaintitle}{Discovering Subgroups\\ with Exceptional Survival Characteristics}
\newcommand{\ourmethod}{\textsc{Sysurv}\xspace}

\newcommand{\RK}{\textsc{RuleKit}\xspace}
\newcommand{\FIBERS}{\textsc{Fibers}\xspace}
\newcommand{\EDS}{\textsc{EsmamDs}\xspace}

\title{\ourmaintitle}
\author{%
  Mhd Jawad Al Rahwanji, Sascha Xu, Nils Philipp Walter, Jilles Vreeken\\
  CISPA Helmholtz Center for Information Security \\
  Saarbr\"ucken, Germany \\
  \texttt{\{jawad.alrahwanji, sascha.xu, nils.walter, jv\}@cispa.de}
}

\begin{document}

\maketitle

\begin{abstract}
  \input{sections/abstract.tex}
\end{abstract}

\input{sections/introduction}
\input{sections/related}
\input{sections/theory}
\input{sections/experiments}
\input{sections/conclusion}

\begin{ack}
  \input{sections/acknowledgements.tex}
\end{ack}

\bibliographystyle{icml2026}
\bibliography{bib/abbreviations,bib/bib-paper}

\newpage
\appendix

\input{sections/appendix.tex}

\end{document}

%% file: macro/defines.tex
\newcommand{\weibull}{\ensuremath{\mathit{Weibull}}}

\pgfplotscreateplotcyclelist{sg-pop-colors}{
	{pr-color-gray},
	{pr-color1a},
}

\pgfplotscreateplotcyclelist{sg-sg-pop-colors}{
	{pr-color1b},
	{pr-color1d},
}

\pgfplotscreateplotcyclelist{ss-rk-pop-colors}{
	{pr-color1a},{pr-color1a, pr-conf-opac}, 
	{pr-color1b},{pr-color1b, pr-conf-opac},
	{pr-color1c},{pr-color1c, pr-conf-opac},
	{pr-color-gray},{pr-color-gray, pr-conf-opac},
}

\pgfplotscreateplotcyclelist{ss-eds-pop-colors}{
	{pr-color1a},{pr-color1a, pr-conf-opac}, 
	{pr-color1c},{pr-color1c, pr-conf-opac},
	{pr-color-gray},{pr-color-gray, pr-conf-opac},
}

\pgfplotscreateplotcyclelist{case-colors2}{
	{pr-color1a},{pr-color1a, pr-conf-opac}, 
	{pr-color1b},{pr-color1b, pr-conf-opac},
	{pr-color-gray},{pr-color-gray, pr-conf-opac},
	{pr-color1c},{pr-color1c, pr-conf-opac},
	{pr-color1d},{pr-color1d, pr-conf-opac},
}

\pgfplotscreateplotcyclelist{case-colors}{
	{pr-color1k},{pr-color1k, pr-conf-opac},
	{pr-color1f},{pr-color1f, pr-conf-opac},
	{pr-color-gray},{pr-color-gray, pr-conf-opac},
	{pr-color1p},{pr-color1p, pr-conf-opac},
	{pr-color1o},{pr-color1o, pr-conf-opac},
}

\pgfplotscreateplotcyclelist{illusrtration-colors}{
	{pr-color1b},
	{pr-color1d},
	{pr-color1f},
	{pr-color-gray},
}

\pgfplotscreateplotcyclelist{box1-colors}{
	{pr-color1a},
	{pr-color1b},
	{pr-color1d},
}

%% file: sections/abstract.tex
In many applications, it is important to identify subpopulations that survive longer or shorter than the rest of the population. In medicine, for example, it allows determining which patients benefit from treatment, and in predictive maintenance, which components are more likely to fail. Existing methods for discovering subgroups with exceptional survival characteristics rely on restrictive assumptions about the survival model (e.g. proportional hazards), require pre-discretized features, and, as they compare average statistics, tend to overlook individual heterogeneity. In this paper, we propose \ourmethod, a non-parametric, fully differentiable method that discovers human-readable rules selecting subgroups with exceptional survival characteristics. Empirical evaluation on a wide range of datasets and settings, including a case study on cancer data, shows that \ourmethod reveals insightful and actionable survival subgroups, outperforming the state of the art.

%% file: sections/introduction.tex
\section{Introduction}

Survival analysis traditionally focuses on estimating whether a \emph{given} group of individuals has different survival characteristics than a reference population. This has applications in many fields, but most obviously in precision medicine, as it allows characterizing patients who benefit from a treatment. What if the groups are not yet known?
Can we \emph{learn} subgroups that stand out in terms of survival characteristics? Can we learn \emph{easily interpretable} rules that select such subgroups?
Can we learn these without restrictive assumptions or pre-discretizing features, allowing even heavy censoring, while keeping individual heterogeneity in mind? 
That is exactly the topic of this paper.

We consider time-to-event data, where each sample $(\mathbf{x}, \delta, t)$ consists of features $\mathbf{x}$, an event indicator $\delta$, and an observed time $t$. The indicator $\delta$ denotes whether the event (e.g. death) occurred. If $\delta = 1$, then $t$ is the event time; if $\delta = 0$, then $t$ is the censoring or last follow-up time. We are interested in learning conjunctive rules composed of conditions on the features (e.g.~$x_1 \in [0.2,1.0] \mathbf{~and~} x_2 \in [0.8,0.9]$). \input{texfigs/example.tex}The rule specifies a subgroup with an exceptional survival characteristic.

For example, on the left in Figure \ref{fig:example}, we show the result of our method, \ourmethod, on neck cancer treatment data.
It learns that ``patients whose tumors have mutations related to DNA repair'' tend to respond very poorly to treatment, which is in line with medical understanding~\citep{barker:2015:tumour}. It is also actionable, as we can test, identify, and treat those patients differently.
On the right-hand side, we consider unemployment data. \ourmethod finds that ``people older than 58 \emph{with} unemployment insurance'' tend to have exceptionally long durations until re-employment (group 1), whereas ``people who previously earned a low wage and \emph{without} unemployment insurance'' tend to find new jobs much more quickly (group 2).

To find these subgroups, we use a \emph{non-parametric} measure of exceptionality based on \emph{individual} survival curves. That is, unlike the state of the art, we do \emph{not} rely on the Cox proportional hazards model, which would restrict us to finding subgroups that have the same functional shape as the overall population. Using individual survival curves also gives valuable information during optimization, i.e., which subjects to include or exclude. We show how to learn rules that single out subgroups with exceptional survival characteristics in a differentiable manner, automatically finding relevant features and intervals in a way that allows easy optimization, yet crisp logical interpretation after learning.

Through an extensive set of experiments on synthetic and real-world data, we show that our method outperforms state-of-the-art methods -- both in terms of recovering the ground truth, and in terms of exceptionality of the found subgroups. Through a case study on treatment of head and neck squamous cell carcinoma, we confirm that \ourmethod rediscovers biomarker associations with good, respectively, poor response to treatment, while also identifying novel subgroups that warrant further investigation.

In a nutshell, our main contributions are as follows
\begin{enumerate}[itemsep=0.1cm,parsep=0.0cm]
    \item We propose a differentiable, non-parametric measure of survival subgroup exceptionality based on individual survival estimates.
    \item We show how to differentiably learn rules that select subgroups with exceptional survival characteristics, and how to post hoc evaluate statistical significance.
    \item We empirically evaluate our method on a wide range of synthetic and real-world datasets, comparing to three state-of-the-art methods.
\end{enumerate}

We make all our code available online at \url{https://eda.group/sysurv}.

%% file: texfigs/example.tex
\begin{wrapfigure}[12]{r}{0.5\textwidth}
    \vspace{-10pt}
    \centering
    \scriptsize
    \begin{tikzpicture}
    \begin{groupplot}[
        group style={
            group size=4 by 1,
            horizontal sep=50pt,
        },
        ytick align=center,
        width=\textwidth/4.1,
        xlabel style={yshift=-5pt},
        ymin=0,
        ymax=1,
        xmin=0,
    ]
        \nextgroupplot[
        pretty line,
        cycle list name=sg-pop-colors,
        xlabel={Follow up time (mo)},
        ylabel style={yshift=5pt},
        xtick={0,25,50,75,100},
        ylabel={Survival probability},
        ]
            \addplot+[very thick] table[x=x, y=y] {data/exps/mmc6_nCounter_PostOp_Population_KM.tsv} node [sloped,below,inner sep=2pt,pos=0.77,rotate=-0] {Population};

            \addplot+[very thick] table[x=x, y=y] {data/exps/mmc6_nCounter_PostOp_3_KM_extended.tsv} node [sloped,below,inner sep=2pt,pos=0.95] {Found subgroup};

        \nextgroupplot[
        pretty line,
        xlabel={Unemployment time (w)},
        ylabel={Unemployment prob.},
        ylabel style={yshift=5pt},
        xtick={0,7,14,21,28},
        cycle list name=sg-sg-pop-colors,
        ]
            \addplot+[very thick] table[x=x, y=y] {data/exps/UnempDur_SySurv_1_KM.tsv} node [sloped,above,inner sep=1pt,pos=0.75,rotate=-0] {Group 1};

            \addplot+[very thick] table[x=x, y=y] {data/exps/UnempDur_SySurv_0_KM.tsv} node [below,inner sep=4pt,pos=0.2,rotate=0] {Group 2};

            \addplot+[very thick, draw=pr-color-gray] table[x=x, y=y] {data/exps/UnempDur_SySurv_2_KM.tsv} node [sloped,above,pr-color-gray,inner sep=3pt,pos=0.5,rotate=5] {Population};
        \end{groupplot}
    \end{tikzpicture}
    \caption{\emph{Survival subgroups}. \ourmethod finds patients suffering from a therapy-resistant tumor \textbf{(Left)}, and people with exceptionally long (1) resp. short (2) durations until re-employment \textbf{(Right)}.}
    \label{fig:example}
\end{wrapfigure}

%% file: sections/related.tex
\section{Related Work}

In this section, we discuss the most relevant related work from survival analysis, survival clustering, subgroup discovery, and survival subgroup discovery.

\textbf{Survival Analysis.} Most ML research on time-to-event data focuses on individual outcome prediction, where ensemble methods--including Random Survival Forests--represent the state of the art~\citep{ishwaran:2008:random,barnwal:2022:survival,bertsimas:2022:optimal,zhang:2024:optimal,huisman:2024:optimal}. Alternatively, \citet{yu:2011:learning,fotso:2018:deep, li:2025:survformer} explored neural networks for estimating survival functions. 
Such models can also naturally incorporate additional modalities, including images and sequences~\citep{cheerla:2019:deep,vale:2021:long,meng:2022:deepmts,meng:2023:merging,saeed:2024:survrnc,farooq:2025:survival}. \ourmethod does not compete with these methods, but rather needs one to obtain individual survival estimates.

\textbf{Survival Clustering} aims to assign subjects to groups with high within-group and low between-group survival similarity, making it a related problem; however, it does not describe or identify these groups. Recent methods use generative models to parameterize latent mixtures of Weibull distributions, \citet{manduchi:2022:deep}, for example, uses variational autoencoders, while \citep{hou:2024:interpretable} employs multilayer perceptrons. Other methods assume proportional hazards \citep{nagpal:2021:deep}, use a multitask framework \citep{cui:2024:deep}, or leverage the latent vector quantization framework \citep{de:2024:survivallvq}, but unlike \ourmethod, these typically make specific assumptions about the underlying distributions or structure, which are then interpreted as data-driven subgroups, a process yielding uninterpretable results compared to rule induction.

\textbf{Subgroup Discovery,} first introduced by \citet{klosgen:1996:explora}, aims to find and describe subpopulations that are exceptional in terms of a target property. Typically, subgroup discovery methods use exceptionality measures tailored to the target variable's data type \citep{song:2016:subgroup,kalofolias:2022:naming}. 
Furthermore, some works make strong assumptions on the distribution of the target variable \citep{friedman:1999:bump,lavravc:2004:subgroup,grosskreutz:2009:subgroup}.
Exceptional model mining \citep{leman:2008:exceptional,duivesteijn:2016:exceptional} extends subgroup discovery to arbitrary target types by measuring the difference between subgroup- and population-level models. While most subgroup discovery methods use combinatorial search, \citet{xu:2024:syflow} showed that gradient-based optimization can efficiently learn subgroups in large datasets without pre-discretization.

\textbf{Survival Subgroup Discovery.} Existing approaches for survival subgroup discovery alter the exceptionality measure, and are typically based on the logrank statistic~\citep{mantel:1966:evaluation}. \RK \citep{gudys:2020:rulekit} relies on heuristic search, while \FIBERS \citep{urbanowicz:2023:fibers} is evolutionary algorithm-based, and, more recently, \EDS \citep{vimieiro:2025:esmamds} was based on Ant Colony Optimization. \citet{relator:2018:survivallamp} explore significant survival subgroup discovery by correcting for multiple testing, albeit at great computational cost.  In contrast, we propose a gradient-descent-optimizable objective that uses flexible survival functions without making assumptions about how subgroup and population survival relate. We also provide statistical guarantees for our subgroups.

%% file: sections/theory.tex
\section{Discovering Survival Subgroups}

We are interested in learning rules that select subgroups with survival trends that are exceptional compared to those of the general population. For that, we need 3 interdependent ingredients: a way to model subgroup survival, a measure for subgroup exceptionality, and a way to learn these rules. We start with notation, and then present each of these ingredients in turn.

\subsection{Notation}

We consider time-to-event (survival) data, where we have a dataset $D=\{(\mathbf{x}^{(i)},t^{(i)},\delta^{(i)})\}^n_{i=1}$ consisting of $n$ i.i.d. realizations, called \emph{subjects}, from a joint distribution $\mathbb{P}(\mathbf{X}, T, \Delta)$. Here, $\mathbf{x} \in \mathbb{R}^p$ is a vector of $p$ covariates that describes the subject, $t \in \mathbb{R}_{\geq0}$ measures time since the start of its observation, and $\delta$ indicates whether the event of interest (e.g.~relapse) has occurred yet. If the event did occur ($\delta = 1$), $t$ is the time at which it was observed. If it did not occur, $t$ is the latest time we observed the subject to still be fine, and so the outcome is said to be \emph{censored}. A survival function $S(t)$ gives the probability of a subject surviving beyond time $t$, i.e.~the complementary cumulative distribution function $\mathbb{P}(T > t)$. To ease notation, we denote the domains of $\mathbf{x}$ by $\mathcal{X}$ and $t$ by $\mathcal{T}$.

A subgroup $Q$ is a set of subjects selected from $D$ using a rule $\sigma_Q\colon\mathcal{X}\to\{0,1\}$ that indicates whether a subject belongs to $Q$ or not, i.e.~$Q=\{(\mathbf{x},t,\delta)\in D\mid\sigma_Q(\mathbf{x})=1\}$. The survival function of subgroup $Q$ is denoted by $\hat{S}_Q(t)$.

\subsection{Subgroup Survival Model}

First, we need to specify a way to model subgroup survival. We adopt the approach in precision medicine, which overcomes the traditional “one-size-fits-all” model by incorporating individual characteristics into treatment outcome predictions \citep{feuerriegel:2024:causal}. This is particularly important in survival analysis where subject heterogeneity can significantly impact outcomes. Thus, this translates to estimating individual survival.

Individual survival functions $\hat{S}(t\mid\mathbf{X=x})$, or $\hat{S}(t\mid\mathbf{x})$ for short, can leverage the covariates $\mathbf{x}$ of a subject and so provide more accurate and more robust survival estimates~\citep{cox:1972:regression} than those estimated marginally~\citep{kaplan:1958:nonparametric}. Specifically important here is that individual survival functions provide fine-grained signal that we can use during optimization, i.e.~to determine which subjects should (not) be part of the subgroup. 
Given individual estimates, the survival function for a subgroup $Q$ selected by rule $\sigma_Q$ is defined as 
\[
    \hat{S}_Q(t)\coloneq\mathbb{E}_{\mathbf{x}\sim P_\mathbf{X}}[\hat{S}(t\mid\mathbf{x})\mid\sigma_Q(\mathbf{x})=1]\;.
\]

To obtain individual survival functions, we fit a non-parametric conditional estimator, or population model, $M$ over the entire dataset. For our instantiation, we use the random survival forest (RSF) \citep{ishwaran:2008:random}. This is a fully-non-parametric continuous-time model that extends the random forest \citep{breiman:2001:random} to time-to-event data. In contrast to well-known parametric models, it does not restrict us to specific survival distributions (e.g.~Weibull) or structures (e.g.~Cox) and allows us to find subgroups with exceptional survival characteristics in general.

\subsection{Subgroup Exceptionality Measure}

Having specified a way to model subgroup survival, we need to define a way to measure the exceptionality of a given subgroup versus a reference group. In traditional survival analysis, the groups of interest are assumed to be \emph{given} and exceptionality is measured on average, at the group level. As we will detail below, this has drawbacks in terms of sensitivity to per-subject deviations when we aim to \emph{learn} the subgroup.

\input{texfigs/illustration.tex}
First, we show in Fig.~\ref{fig:illustration} (left) how individual estimates can \emph{reveal} differences from the reference survival that group-level estimates \emph{obscure}. The dark green line shows the group-level average over four individuals (light green lines). The average curve obscures that the curves of subjects 1 and 2 are very different, i.e.~delayed, compared to those of subjects 3 and 4, effectively underestimating the true exceptionality. In general, individual-level exceptionality measures offer increased sensitivity to deviations of individual survival functions, such as those crossing the reference, making them particularly suited for optimization of subgroup membership rules.

Independent of whether we measure exceptionality at the group or individual level, it is important to consider the absolute difference to the reference. We give an example in Fig.~\ref{fig:illustration} (right). While it is trivial to identify groups 1 and 2 as exceptional even when we measure relative to the reference, the large exceptionality of group 3 only becomes clear when we measure the difference in absolute terms.

Next, we formalize these intuitions, and show that considering absolute differences of individual estimates to the reference gives us more signal than the absolute difference in group averages. 

\begin{restatable}{proposition}{Individual}
    \label{prop:individual}
    Given two groups $A$ and $B$, selectable by rule $\sigma_A$ and $\sigma_B$, respectively, for which the expected group-level survival at any time $t$ are $\hat{S}_A(t)$ and  $\hat{S}_B(t)$, and for which individual-level survival is denoted by $\hat{S}(t\mid\mathbf{x})$. The expected absolute difference in survival of the subjects selectable by $\sigma_A$ is as, or more, sensitive than the absolute difference of $\hat{S}_A(t)$ from $\hat{S}_B(t)$,
    \[
        \mathbb{E}_{\mathbf{x}\sim P_\mathbf{X}}[\ell^1_t(s(\mathbf{x}),s_B)\mid\sigma_A(\mathbf{x})=1]\geq\ell^1_t(s_A,s_B)\;,
    \]
    where $\ell^1_t(\cdot,\cdot)$ is an absolute difference measure, and we write $s_\circ$ for $\hat{S}_\circ(t)$, and $s(\mathbf{x})$ for $\hat{S}(t\mid\mathbf{x})$.
\end{restatable}

We provide the proof, as well as an informative Corollary, in Appx.~\ref{apx:proof}. For our purposes, $s_A$ would typically belong to a group of interest, e.g.~a subgroup, and $s_B$ to a reference group, e.g.~the entire dataset. To measure the exceptionality between two survival estimates $s_A,s_B\geq0$ at some time $t$, we consider the $L^1$ distance $\ell^1_t(s_A,s_B)\coloneq|s_A-s_B|$.

To capture the survival interplay between groups, we leverage individual estimates over the entire time domain $\mathcal{T} = [0, t_{\max}]$ by extending the absolute difference measure $\ell^1_t(\cdot,\cdot)$ for two survival functions $S_A,S_B$ as $\ell^1_\mathcal{T}(S_A,S_B)\coloneq\int_{t\in\mathcal{T}}\left|S_A(t)-S_B(t)\right|dt$. Finally, we define our exceptionality measure as the expected difference in survival of the subjects of a subgroup selectable by $\sigma$ from the estimated survival in the population $\hat{S}_D(t)$ throughout $\mathcal{T}$ as
\begin{equation}
    \label{eq:group_abs_dev}
    \phi(\sigma,\sigma_D)\coloneq\mathbb{E}_{\mathbf{x}\sim P_\mathbf{X}}[\ell^1_\mathcal{T}(\hat{S}(t\mid\mathbf{x}),\hat{S}_D(t))\mid\sigma(\mathbf{x})=1]\;.
\end{equation}
In our experiments, we will compare to methods that rely on the logrank for measuring exceptionality, a signed, group-level statistic, widely used in survival analysis.

\subsection{Rule Learner}

Armed with our subgroup survival model and exceptionality measure, we now present \ourmethod for learning rules that select subgroups with exceptional survival characteristics, via gradient-descent.

\textbf{Learnable Rules.} Subgroups are   selected via rules, so we first formalize the language of these rules. Traditionally, a subgroup is defined by a hard rule $\sigma\colon\mathbf{x}\mapsto\bigwedge^p_{j=1}\pi(x_j;\alpha_j,\beta_j)$, where each $\pi$ is a Boolean condition evaluating to true (1) if a covariate $x$ falls within the  interval defined by lower and upper bounds $\alpha,\beta\in\mathbb{R}$, e.g.~\emph{“$18 < \text{age} < 32$”}. To integrate rule induction into a gradient-based optimization pipeline, we employ a continuous relaxation of these logical expressions.

To facilitate differentiable rule induction, we use soft rules $\hat{\sigma}\colon\mathcal{X}\to[0,1]$ consisting of soft conditions $\hat{\pi}\colon\mathbb{R}\to[0,1]$. These soft conditions model the probability of a covariate being inside the specified interval via a composition of two opposing sigmoids located at the learnable bounds $\alpha$ and $\beta$. A temperature hyperparameter $\tau>0$ controls the strictness of the bounds; as $\tau\to 0$, the soft condition $\hat{\pi}$ converges to a Boolean interval (Fig.~\ref{fig:box1}).

Following \citet{xu:2024:syflow}, we define the soft condition as
\[
    \hat{\pi}(x_j;\alpha_j,\beta_j,\tau)\coloneq\frac{e^{\frac{1}{\tau}(2x_j-\alpha_j)}}{e^{\frac{1}{\tau}x_j}+e^{\frac{1}{\tau}(2x_j-\alpha_j)}+e^{\frac{1}{\tau}(3x_j-\alpha_j-\beta_j)}}\;,
\]
and the differentiable rule learner using a weighted harmonic mean of soft conditions as 
\[
    \hat{\sigma}(\mathbf{x}; \boldsymbol{\alpha}, \boldsymbol{\beta}, \mathbf{w}, \tau)\coloneq\frac{\sum^p_{j=1}w_j}{\sum^p_{j=1}w_j\;\hat{\pi}(x_j;\alpha_j,\beta_j,\tau)^{-1}}\;,
\]
where $\mathbf{w}\in\mathbb{R}^p$ represents a vector of learnable weights. These weights allow the model to perform feature selection, i.e.~a covariate $x$ is actively included in the conjunction when $w>0$ and effectively ignored as $w\to0$. The harmonic mean structure is chosen because it serves as a smooth approximation of the logical-and operator; if any single active condition $\hat{\pi}$ approaches zero, the entire rule output $\hat{\sigma}$ tends toward zero. Henceforth, the output $\hat{\sigma}(\mathbf{x})$ can be interpreted as the probability $\mathbb{P}(\hat{\sigma}(\mathbf{X})=1\mid\mathbf{X=x})$ that a subject belongs to the subgroup (yellow box in Fig.~\ref{fig:box2}).

\textbf{Objective.} Our goal is to learn those rules to select the subgroups that maximize our objective. For rule updates, we need to take the gradient of our objective w.r.t the parameters of the soft rule $\hat{\sigma}$. Hence, we rewrite the expectation in Eq.~\eqref{eq:group_abs_dev} to derive our objective where $\hat{\sigma}$ explicitly appears as
\begin{align}
    \label{eq:diff_group_abs_dev}
    &\phi(\hat{\sigma},\hat{\sigma}_D)=\int_{\mathbf{x}\in\mathcal{X}}\ell^1_\mathcal{T}(\hat{S}(t\mid\mathbf{x}),\hat{S}_D(t))\;\frac{\mathbb{P}(\mathbf{X}=\mathbf{x})}{\mathbb{P}(\hat{\sigma}(\mathbf{X})=1)}\;\hat{\sigma}(\mathbf{x})\;d\mathbf{x}\;,
\end{align}
by using the expected value definition, marginalization rule and Bayes' theorem. We estimate the integral over $\mathcal{X}$ and $\mathcal{T}$ in Eq.~\eqref{eq:diff_group_abs_dev} using the standard Monte Carlo and trapezoidal methods, respectively. Our (core) objective is defined as 
\[
    \hat{\phi}(\hat{\sigma},\hat{\sigma}_D)\coloneq\frac{1}{|\hat{\sigma}|}\sum_{i=1}^n\ell^1_\mathcal{T}(\hat{S}(t\mid\mathbf{x}^{(i)}),\hat{S}_D(t))\;\hat{\sigma}(\mathbf{x}^{(i)};\boldsymbol{\theta})\;,
\]
where the subgroup size $|\hat{\sigma}|=\sum^n_{i=1}\hat{\sigma}(\mathbf{x}^{(i)};\boldsymbol{\theta})$, and $\boldsymbol{\theta}$ stands for $\boldsymbol{\alpha}$, $\boldsymbol{\beta}$, and $\mathbf{w}$, collectively.

\input{texfigs/box.tex}

To avoid learning overly specific subgroups, and as is standard in subgroup discovery, we introduce a size penalty $|\hat{\sigma}|^\gamma$ to our objective. This allows us to control the trade-off between exceptionality and subgroup size via a hyperparameter $\gamma\in[0,1]$. To learn non-redundant subgroups, we sequentially learn consecutive rules that describe subgroups whose survival is simultaneously exceptional w.r.t the population and the set of $q$ preceding subgroups. For that, we add a regularizing term to our objective that ensures that we get a diverse set of subgroups. Thus, \[
    \arg\max_{\boldsymbol{\theta}}\left[|\hat{\sigma}|^\gamma\;\hat{\phi}(\hat{\sigma},\hat{\sigma}_D)+\sum^q_{g=1}|\hat{\sigma}|^\gamma\;\hat{\phi}(\hat{\sigma},\hat{\sigma}_g)\right]\!,
\]
gives our diversity-ensuring, size-aware (full) objective.

\textbf{Optimization.} Gradient-based optimization allows us to efficiently learn the parameters of our rules, and therewith scale \ourmethod to large datasets. For every subject, each feature is subjected to its corresponding soft condition according to the learned bounds, which are then weighted and combined into a soft rule that predicts a membership probability. We iteratively learn the weights and bounds using standard first-order gradient-based optimization techniques \citep{kingma:2015:adam} for a number of epochs, while annealing the temperature to arrive at crisp discretizations. For the pseudocode and the annealing schedule see Appx.~\ref{apx:algorithm}.

\textbf{Rule Sparsity.} In correlated data, there often exist different ways (i.e.~different conditions over different features) to select the same subgroup. Some of these features may be easier to measure, others may be causally more relevant. As the goal of \ourmethod is exploratory, the default is to not sparsify the rule but rather let the user decide which features (terms) are most informative. In specific cases, such as when dealing with high-dimensional data, the user may instead wish a small, non-redundant set of terms. To this end, we propose a post-hoc strategy, in which we iteratively remove that condition from the rule that minimally changes which subjects are selected, and stop when a minimal Jaccard similarity is reached. We present the pseudocode in Appx.~\ref{apx:pruning}.

\subsection{Significance Testing}

In critical applications, e.g.~healthcare, it is vital to statistically validate the results. While formal guarantees for continuous optimization are challenging, we provide post-hoc statistical guarantees via permutation-tests. \citet{duivesteijn:2011:exploiting} propose to create independent copies of the input dataset $D$ in which the dependency between the covariates, $\mathbf{x}$, and the outcomes, ($t, \delta$), is broken by random permutation. Running \ourmethod on each of these copies enables building a distribution of exceptionalities from false discoveries. Then, by applying the central limit theorem, we can determine the p-value of the exceptionality of the subgroup(s) we find in $D$ from their Z-scores. We use an alpha cutoff of 0.05. When discovering multiple subgroups, we use Bonferroni correction \citep{dunn:1961:multiple} to correct for multiple hypothesis testing. For our instantiation of the above, see Appx.~\ref{apx:validation}.

%% file: texfigs/illustration.tex
\begin{wrapfigure}[17]{r}{0.5\textwidth}%
    \vspace{-10pt}
    \centering
    \scriptsize
    \begin{tikzpicture}
    \begin{groupplot}[
        group style={
            group size=4 by 1,
            horizontal sep=50pt,
        },
        ytick align=center,
        width=\textwidth/4.1,
        xlabel style={yshift=-5pt},
        ymin=0,
        ymax=1,
        xmin=0,
    ]
        
        \nextgroupplot[
        pretty line,
        ylabel={Survival prob.},
        ylabel style={yshift=5pt},
        xlabel={Survival time},
        cycle list name=illusrtration-colors,
        legend image post style={xscale=0.5, ultra thick},
        legend style={
            at={(0.15,1.2)},
            anchor=south,
            legend columns=3,
        },
        legend entries={$\hat{S}(t|\mathbf{x}_i)$,,,, $\hat{S}(t)$, Pop.}
        ]
            \addplot+[very thick, draw=pr-color1a, opacity=0.34] table[x=x, y=y] {data/exps/indiv_survival_5.tsv} node [above,inner sep=2pt,pos=0.5,rotate=-0, color=pr-color1a] {1};

            \addplot+[very thick, draw=pr-color1a, opacity=0.34] table[x=x, y=y] {data/exps/indiv_survival_1.tsv} node [below,inner sep=6pt,pos=0.13,rotate=60, color=pr-color1a] {2};

            \addplot+[very thick, draw=pr-color1a, opacity=0.34] table[x=x, y=y] {data/exps/indiv_survival_7.tsv} node [above,inner sep=8pt,pos=0.3,rotate=45, color=pr-color1a] {3};

            \addplot+[very thick, draw=pr-color1a, opacity=0.34] table[x=x, y=y] {data/exps/indiv_survival_6.tsv} node [below,inner sep=2pt,pos=0.13,rotate=-0, color=pr-color1a] {4};

            \addplot+[very thick, draw=pr-color1a] table[x=x, y=y] {data/exps/above_survival.tsv};

            \addplot+[very thick, draw=pr-color-gray] table[x=x, y=y] {data/exps/pop_survival_indivs.tsv};
        
        \nextgroupplot[
        pretty line,
        cycle list name=illusrtration-colors,
        legend image post style={xscale=0.5, ultra thick},
        xlabel={Survival time},
        ylabel={Survival prob.},
        ylabel style={yshift=5pt},
        legend style={
            at={(-0.95,1.2)},
            anchor=south west,
            legend columns=4,
        },
        legend entries={$\hat{S}_1(t)$, $\hat{S}_2(t)$, $\hat{S}_3(t)$, Pop.}
        ]
            \addplot+[very thick] table[x=x, y=y] {data/exps/above_survival.tsv};

            \addplot+[very thick] table[x=x, y=y] {data/exps/below_survival.tsv};

            \addplot+[very thick] table[x=x, y=y] {data/exps/crossing_survival.tsv};

            \addplot+[very thick] table[x=x, y=y] {data/exps/pop_survival.tsv};

        \end{groupplot}
    \end{tikzpicture}
    \caption{\textbf{(Left)} Individual survival functions $\hat{S}(t|\mathbf{x}_i)$ are more informative than their group-level estimate $\hat{S}(t)$. \textbf{(Right)} The absolute difference between group and reference survival (Pop.) is effective in the absence of crossing ($\hat{S}_1$ and $\hat{S}_2$) and in its presence ($\hat{S}_3$), unlike the signed difference.}
    \label{fig:illustration}
\end{wrapfigure}%

%% file: texfigs/box.tex
\begin{wrapfigure}[15]{r}{0.5\textwidth}
    \vspace{-20pt}
    \centering
    \scriptsize
    \subfloat[Soft condition]{%
        \centering
        \begin{tikzpicture}
        \begin{groupplot}[
            group style={
                group size=4 by 1,
                horizontal sep=55pt,
            },
            ytick align=center,
            width=\textwidth/4.5,
            xlabel style={yshift=-5pt},
            ymin=0,
            ymax=1,  
            xmin=0,
        ]
            
            \nextgroupplot[
            pretty line,
            ylabel = {$\hat{\pi}(x;\tau)$},
            xlabel={$x$},
            ylabel style={yshift=5pt},
            xtick={0,0.3,0.7,1},
            cycle list name=box1-colors,
            ytick={0,0.3,0.7,1},
            legend image post style={xscale=0.5, ultra thick},
            legend style={
                at={(0.35,1.2)},
                anchor=south,
                legend columns=3,
            },
            legend entries={$0.1$,$0.05$,$0.025$}
            ]
                \foreach \i in {1, 2, 3} {
                    \addplot+[very thick] table[x=X, y=y\i, col sep=comma] {data/exps/box_lines.csv};
                }
            
            \end{groupplot}
        \end{tikzpicture}
        \label{fig:box1}
    }
    \hfill
    \subfloat[Soft rule]{%
        \centering
        \begin{tikzpicture}
        \node at (0.5,0.55) {\includegraphics[width=2cm]{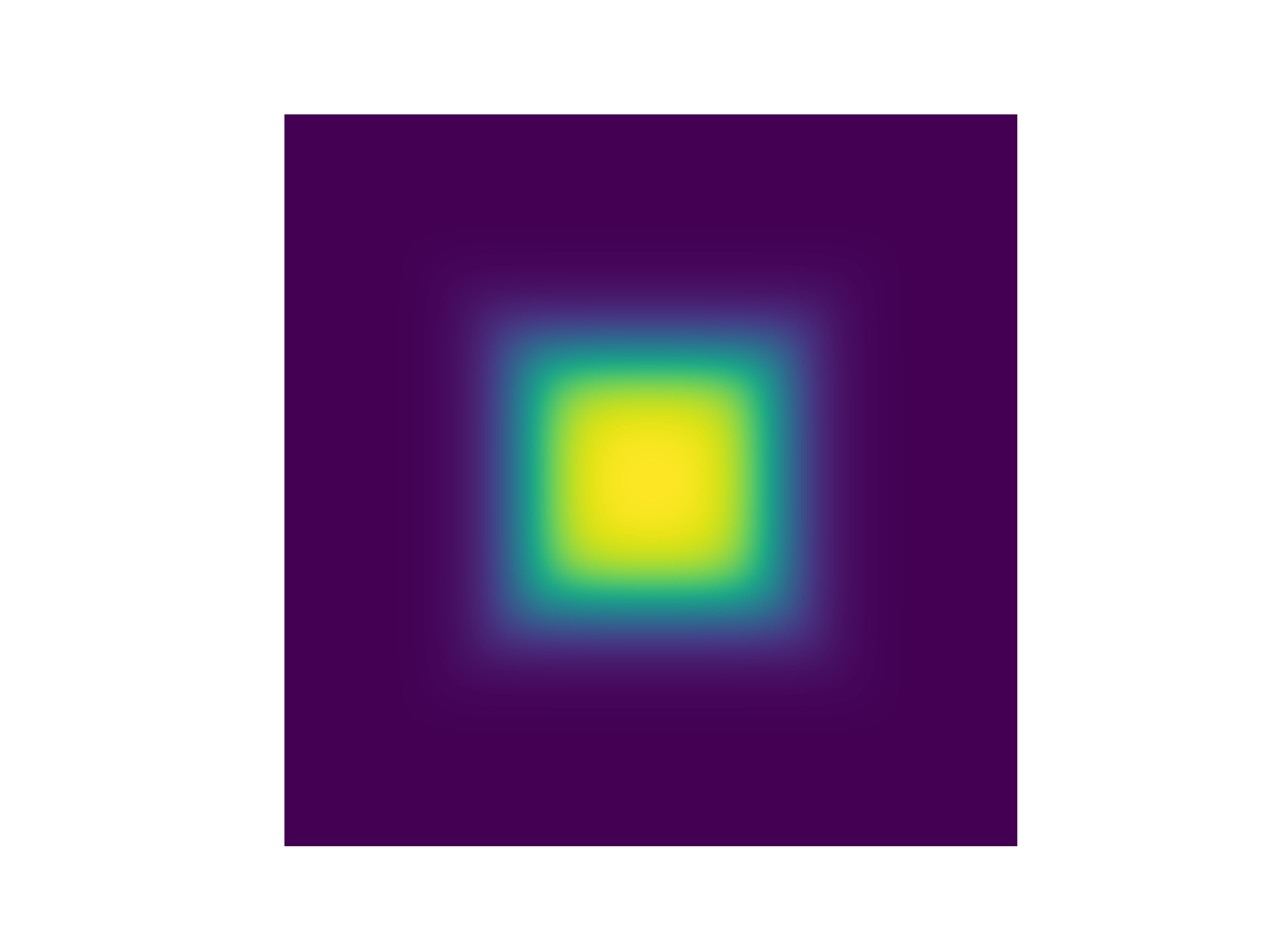}};
        \begin{groupplot}[
            group style={
                group size=4 by 1,
                horizontal sep=55pt,
            },
            ytick align=center,
            width=\textwidth/4.5,
            xlabel style={yshift=-5pt},
            ymin=0,
            ymax=1,  
            xmin=0,
        ]
            
            \nextgroupplot[
            pretty line,
            axis equal image,
            xtick = {0,0.5,1},
            ytick = {0,0.5,1},
            ylabel style={yshift=5pt},
            xlabel = {$x_1$},
            ylabel = {$x_2$},
            xmin=0,
            xmax=1,
            ymin=0,
            ymax=1,
            colorbar horizontal,
            colormap name=viridis,
            colorbar style={
                width=2.7cm,
                height=0.1cm,
                xtick={0,1},
                at={(0.45,1.4)},
                anchor=south,
            },
            point meta min=0,
            point meta max=1
            ]

            \end{groupplot}
        \end{tikzpicture}
        \label{fig:box2}
    }
    \caption{\emph{Soft rules.} The soft condition approaches a hard interval with decreasing temperature $\tau\to 0$ \textbf{(a)}. Multiple soft conditions combine to form a soft rule depicted as a hyper box in the covariate space \textbf{(b)}. Adapted from \citet{xu:2024:syflow}.}
    \label{fig:box}
\end{wrapfigure}

%% file: sections/experiments.tex
\section{Experiments}
 
Next, we empirically evaluate \ourmethod on synthetic and real-world data. We compare it to state-of-the-art methods \RK~\citep{gudys:2020:rulekit}, \FIBERS~\citep{urbanowicz:2023:fibers}, and \EDS~\citep{vimieiro:2025:esmamds}. All baselines optimize the logrank for exceptionality, but do so employing different search strategies. We use the implementations of the respective authors, and allow each method, including \ourmethod, up to 48 hours per experiment. We tune hyperparameters for all methods using grid search (see Appx.~\ref{apx:hyperparams}). We give the pseudocode for our data generator and additional results in Appx.~\ref{apx:data} and \ref{apx:additional}.

\subsection{Synthetic Data}

We start by evaluating on data with known ground truth. To this end, we generate data as follows. We sample $p=10$ feature variables $\mathbf{X}_j$ from a uniform distribution $\mathcal{U}(0,1)$, and the survival time $T$ from a Weibull distribution for its configurability, $\weibull(1.5,5)$. Subject to $10\%$ uniform censoring, we create a dataset of $n=10\,000$ subjects. We randomly choose $k \leq p$ features (here, $k=2$) and sample corresponding bounds, creating a rule that covers 20\% of the feature space. We randomly split the subjects into those to be included in the subgroup (20\%) and those not (80\%). We ensure the feature values for the samples designated not to belong to the subgroup to be outside the rule bounds across conditional features. We then do the same to the subjects designated to belong to the subgroup so they fulfill the rule but also resample their outcomes from a Weibull distribution having the same shape of $1.5$ as the overall population, but a different scale of $1$. In other words, subjects experience their event $5\times$ earlier in the planted subgroup than outside, on expectation. With our benchmark, we can vary the data or the subgroup generation parameters.

\input{texfigs/synth_panel_bands.tex}

\textbf{Scalability.} First, we assess the performance of \ourmethod and its competitors on data where we vary the number of features. We report the average $F_1$-scores over 10 runs in Fig.~\ref{fig:synth_panel_bands} (left). We see that all methods are stable in terms of F1-scores as the dimensionality increases, and that \RK does not finish in time for more than 500 features. \ourmethod outperforms the closest competitor, \EDS, by a wide margin. We provide the runtimes in Fig.~\ref{fig:additional_synth_panel} (left) in Appx.~\ref{apx:additional}.

\textbf{Censoring.} Next, we assess robustness to censoring when increasing the percentage of censored subjects. We give the results in Fig.~\ref{fig:synth_panel_bands} (center); we see that \ourmethod performs very well, with $F_1$ of 0.8 and above for up to 60\% censoring. We check whether having the subgroup outsurvive the population, and vice versa, affects performance by varying the subgroup and fixing the population hazards, and give the results in Fig.~\ref{fig:additional_synth_panel} (center).

\textbf{Subgroup Size.} Lastly, we assess retrieval quality as the percentage of subjects that belong to the planted subgroup increases. We see in Fig.~\ref{fig:synth_panel_bands} (right) that all methods improve as the planted subgroup size increases, with \ourmethod performing the best. We attribute the improvement in the baselines to them generally preferring larger subgroups, according to our observations (Appx.~\ref{apx:additional}). In Fig.~\ref{fig:additional_synth_panel} (right), we also provide results on sample sufficiency by varying dataset sizes.

Overall, \ourmethod is  stable with average $F_1 \approx0.8$, up to very high rates of censoring ($>80\%$). \EDS is the closest competitor, hovering around 0.55, closely followed by \RK consistently at 0.45, while \FIBERS follows in all experiments, with average $F_1 \approx0.35$ and large standard errors.

\begin{table*}[b!]
    \centering
    \scriptsize
    \caption{\emph{Real-world setting.} Exceptionality of subgroups discovered by \ourmethod (\emph{Ours}), \RK (\textsc{RK}), \EDS (\textsc{ED}), and \FIBERS (\textsc{Fib}), measured using 3 metrics (our objective, logrank, and mean-shift). \ourmethod ranks first across metrics. Greater is better. Highest values marked in bold.}
    \label{tab:real}
    \begin{tabular}{lrrrrrrrrrrrrr}
        \toprule
        & \multicolumn{4}{c}{Our objective} & \multicolumn{4}{c}{Logrank} & \multicolumn{4}{c}{Mean-shift} \\
        
        \cmidrule(lr){2-5} \cmidrule(lr){6-9} \cmidrule(lr){10-13}
        
        & \emph{Ours} & \textsc{RK} & \textsc{ED} & \textsc{Fib} & \emph{Ours} & \textsc{RK} & \textsc{ED} & \textsc{Fib} & \emph{Ours} & \textsc{RK} & \textsc{ED} & \textsc{Fib} \\
        \midrule
        UnempDur & \textbf{6.4} & 3.8 & 3.8 & 3.8 & 19.3 & \textbf{226.5} & 125.0 & 125.0 & 2.0 & \textbf{2.1} & \textbf{2.1} & \textbf{2.1} \\
        Nwtco & \textbf{1011.6} & 681.0 & 681.0 & 208.6 & \textbf{374.2} & 281.3 & 31.4 & 0.0 & \textbf{1162.0} & 688.2 & 688.2 & 0.6 \\
        Rott2 & \textbf{559.2} & 326.2 & 487.6 & 487.6 & \textbf{267.0} & 93.0 & 71.7 & 169.3 & \textbf{31.5} & 10.0 & 25.2 & 25.2 \\
        Rdata & \textbf{259.2} & 213.8 & 213.8 & 158.5 & \textbf{113.4} & 107.8 & 14.6 & 43.3 & \textbf{1116.6} & 933.5 & 933.5 & 705.5 \\
        Aids2 & \textbf{66.6} & 60.5 & 54.0 & 34.7 & 0.2 & 1.7 & \textbf{4.3} & 0.0 & 22.7 & 61.2 & \textbf{61.4} & 0.7 \\
        Dialysis & 4.5 & \textbf{4.7} & \textbf{4.7} & 4.6 & 6.4 & 68.4 & 70.9 & \textbf{305.9} & 0.6 & \textbf{5.7} & 3.1 & 4.5 \\
        TRACE & 261.8 & 272.5 & \textbf{332.1} & 261.6 & 0.0 & \textbf{94.6} & 93.2 & 0.0 & 0.0 & 0.6 & \textbf{1.7} & 0.0 \\
        Support2 & \textbf{132.8} & 81.4 & 104.9 & 54.8 & \textbf{251.0} & 59.9 & 46.8 & 11.9 & \textbf{405.7} & 227.6 & 405.5 & 71.0 \\
        DataDIVAT2 & \textbf{286.3} & 141.3 & 172.5 & 141.3 & \textbf{33.6} & 31.6 & 10.9 & 31.6 & \textbf{3.2} & 1.6 & 2.1 & 1.6 \\
        ProstateSurv. & \textbf{20.1} & 7.5 & 14.1 & 5.4 & \textbf{425.7} & 64.7 & 82.5 & 82.5 & \textbf{6.1} & 3.7 & 3.4 & 1.0 \\
        Actg & \textbf{34.6} & 24.2 & 18.8 & 11.3 & \textbf{60.0} & 1.1 & 12.3 & 0.0 & \textbf{54.1} & 27.3 & 6.0 & 0.4 \\
        Scania & \textbf{172.5} & 147.4 & 147.4 & 84.7 & 4.7 & \textbf{14.7} & 13.7 & 9.9 & \textbf{6.5} & 2.4 & 2.4 & 0.1 \\
        Grace & \textbf{38.6} & 16.5 & 21.0 & 13.1 & \textbf{59.3} & 33.3 & 28.7 & 13.8 & \textbf{64.0} & 16.0 & 21.8 & 11.1 \\
        \midrule
        Avg. rank & \textbf{1.38} & 2.65 & 2.27 & 3.69 & \textbf{2.04} & 2.19 & 2.69 & 3.08 & \textbf{1.81} & 2.46 & 2.23 & 3.5 \\
        \bottomrule
    \end{tabular}
\end{table*}

\subsection{Real-World Data}

Next, we evaluate how well \ourmethod performs on real-world data. To this end, we consider 13 time-to-event datasets from the SurvSet repository \citep{drysdale:2022:survset} that span different domains and have at least 1000 subjects and 7 covariates. Because, for these datasets, the ground truth is unknown, we focus on the exceptionality of the found subgroups using 3 metrics: logrank, mean-shift, and our objective, defined in Appx.~\ref{apx:exceptionality}. As before, we compare \ourmethod to \RK, \EDS, and \FIBERS, where we consider the best of the top-5 subgroups found by each baseline, and the first subgroup found by \ourmethod.  We report the results in Table~\ref{tab:real}. For mean-shift, we are only interested in the magnitude of the shift, and hence report absolute values. We provide the average subgroup sizes and rule lengths found by each method along with the runtimes in Tables~\ref{tab:subgroup_sizes}, \ref{tab:rule_lengths}, and \ref{tab:runtimes}, respectively.

We see that, despite them having an advantage, \ourmethod outperforms its competitors by a wide margin. For most datasets it achieves (much) higher scores, and overall the best average ranks, for each metric. For the \emph{Nwtco}, \emph{Rott2}, \emph{ProstateSurvival}, \emph{Actg}, and \emph{Grace} datasets, it even scores between 14.7\% to 1091.7\% better on \emph{each} metric.

For some datasets, e.g.~\emph{Unemployment Duration}, we find that methods excel at their own metrics. To investigate this irrespective of these metrics, we consider the Kaplan-Meier survival curves \citep{kaplan:1958:nonparametric} of the survival subgroups found by \ourmethod and \RK on \emph{Unemployment Duration} in Fig.~\ref{fig:real1}. We see that both found subgroups with substantially longer time until re-employment than the overall population. The rule \RK found, \emph{“unemployment insured”}, is very succinct. The rule found by \ourmethod, \emph{“$\text{age} > 47.43 \land \text{replacement rate} \in [0.11,1.95] \land \text{disregard rate} < 0.96 \land \text{tenure} > 6.14 \land \text{log(wage)} \in [4.20,7.36]$”}, is also informative and more detailed as it describes that subjects who are relatively old and earned a high wage tend to take longer until re-employment.

\input{texfigs/real.tex}

Here, we explore exceptionality using grid search over the same lattice of $\gamma$ values for all datasets. On the heart attack dataset, \emph{TRACE}, \ourmethod finds a very large subgroup. By adjusting the hyperparameter $\gamma$, we can explore smaller subgroups. When we set its value to 0.025, we get a subgroup that at 470 subjects covers more than half the subgroup of 896 subjects found by \EDS, scoring 303.5 on our objective, 97.6 on logrank (new best), and 1.4 on mean-shift. When we inspect the Kaplan-Meier curves of the subgroups in Fig.~\ref{fig:real2}, we see that \ourmethod in fact outperforms \EDS by a margin. Inspecting the found rules themselves, 
\emph{“$\lnot \text{clinical heart failure}$”} for \EDS, and
\emph{“$\text{wall motion index} > 1.37 \land \text{female} \land \lnot \text{clinical heart failure}$”} for \ourmethod, we see that \ourmethod finds a more detailed rule.

Overall, these experiments show that \ourmethod performs well on a wide range of settings, and remarkably well compared on the same metrics that the competitors optimize for. It finds sensible and precisely described subgroups that are generally more exceptional than those found by the closest competitors, while allowing user-adjustable subgroup sizes.

\subsection{Case Study: Neck Cancer}

Last, in collaboration with domain experts, we qualitatively evaluate \ourmethod via a case study. We retrospectively consider locally advanced head and neck squamous cell carcinoma (HNSCC) patients undergoing radiochemotherapy~\citep{schmidt:2020:comparison}. The treatment data consists of two cohorts. The first, or primary, cohort are patients who only received radiochemotherapy ($n=136$). The second, or postoperative, cohort are patients who received radiochemotherapy after a tumor removal operation ($n=190$). The covariates are tumor gene expressions ($p=158$) and the outcome is the time until tumor recurrence. Censoring is 60\% for the first and 85\% for the second cohort. From a clinical perspective, we are both interested in subgroups of subjects that respond better, and those that respond worse to treatment, as this allows more targeted (personalized) treatment.

We run \ourmethod on both datasets and discover four subgroups each. We show the survival curves in Fig.~\ref{fig:study} and give the rules in Tables~\ref{tab:prune} and \ref{tab:prune2} in Appx.~\ref{apx:additional}. For both cohorts, \ourmethod finds two subgroups that respond better, and two that respond worse, to the respective treatments than the overall populations. Undergoing the tumor removal operation expectedly improves the overall survival, nevertheless, subgroup discovery is perfectly suited, when RCTs may fall short, for pointing out when this is not the case, and when it is even detrimental.

First, consider the primary cohort. We see in Fig.~\ref{fig:study1} that approximately half of the population are still alive after 80 months. Subgroups $S_1$ (yellow, n=6) and $S_3$ (orange, n=28) identify subjects that respond better, while subgroups $S_2$ (red, n=12) and $S_4$ (blue, n=25) respond worse. Although the effect sizes are not significant under our test, it is highly encouraging that the rules select on meaningful biomarkers. \citet{schmidt:2020:comparison} recently reported on the survival profile when stratifying primary radiochemotherapy patients based on hypoxia 15- and 26-gene expressions. Subgroup $S_4$ not only shows a similar profile, but like subgroup $S_3$ also mainly selects on hypoxia-associated genes.

\input{texfigs/study.tex}

Next, we investigate the results on the postoperative cohort. Here, almost 80\% of patients survive longer than 80 months (Fig.~\ref{fig:study2}). We are specifically interested in subgroups $S_5$ (green, n=12) and $S_7$ (purple, n=12) as these identify subjects responding very poorly. Both are significant under our statistical model. As before, the corresponding rules make biological sense: they both select on genes relating to cellular communication, metabolism, and DNA repair, which are related to highly aggressive tumor subtypes~\citep{lendahl:2009:generating,toustrup:2011:development} for which conventional treatment not only fails to improve patient condition but even promotes recurrence \citep{barker:2015:tumour}. The learned rules allow clinicians to pre-screen these subjects and offer them alternative treatment. 

%% file: texfigs/synth_panel_bands.tex
\begin{figure*}
    \centering
    \scriptsize
    \begin{tikzpicture}
    \begin{groupplot}[
        group style={
                group size=4 by 1,
                horizontal sep=55pt,
            },
            width=\textwidth/3.85,
            ylabel style={yshift=5pt},
            xlabel style={yshift=-5pt},
            ylabel={$F_1$},
            pretty grid line,
            ytick={0.2,0.4,0.6,0.8},
            extra y ticks={0,0,1,0},
            extra y tick style={grid=none},
            ymin=0,
            ymax=1,
            cycle list name=pr-colors-conf,
        ]

        \nextgroupplot[
            pretty line,
            pretty fill legend,
            xlabel={Number of features $p$},
            xlabel style={xshift=-30pt},
            ylabel style={yshift=30pt},
            xmode=log,
            log basis x=10,
            scaled x ticks=base 10:1,
            cycle list name=pr-colors-conf,
        ]
            \foreach \i in {sysurv, rulekit, esmamds, fibers} {
                \addplot+[very thick] table[x=x, y=y] {data/exps/1kdims_f1_\i.tsv};
                \addplot+[forget plot, draw=none, name path=low] table[x=x, y=y_c0] {data/exps/1kdims_f1_\i.tsv};
                \addplot+[forget plot, draw=none, name path=up] table[x=x, y=y_c1] {data/exps/1kdims_f1_\i.tsv};
                \addplot fill between[of=low and up];
            }

        \nextgroupplot[
            pretty line,
            pretty fill legend,
            xlabel={Percentage of censored subjects},
            cycle list name=pr-colors-conf,
            xtick={0,0.3,0.6,0.9},
        ]
            \foreach \i in {sysurv, rulekit, esmamds, fibers} {
                \addplot+[very thick] table[x=x, y=y] {data/exps/censorship_f1_\i.tsv};
                \addplot+[forget plot, draw=none, name path=low] table[x=x, y=y_c0] {data/exps/censorship_f1_\i.tsv};
                \addplot+[forget plot, draw=none, name path=up] table[x=x, y=y_c1] {data/exps/censorship_f1_\i.tsv};
                \addplot fill between[of=low and up];
            }

        \nextgroupplot[
            pretty line,
            pretty fill legend,
            xlabel={Percentage of subjects in subgroup},
            cycle list name=pr-colors-conf,
            legend style={
                at={(1.4,0.4)},
                anchor=west,
                legend columns=1,
            },
            legend entries={\ourmethod,,\RK,,\EDS,,\FIBERS},
        ]
            \foreach \i in {sysurv, rulekit, esmamds, fibers} {
                \addplot+[very thick] table[x=x, y=y] {data/exps/target_size_f1_\i.tsv};
                \addplot+[forget plot, draw=none, name path=low] table[x=x, y=y_c0] {data/exps/target_size_f1_\i.tsv};
                \addplot+[forget plot, draw=none, name path=up] table[x=x, y=y_c1] {data/exps/target_size_f1_\i.tsv};
                \addplot fill between[of=low and up];
            }
        \end{groupplot}
    \end{tikzpicture}
    \caption{\emph{Synthetic setting.} Comparison of \ourmethod and each of \RK, \EDS, and \FIBERS in terms of $F_1$-scores recovering planted subgroups with increasingly large dataset dimensionalities \textbf{(Left)}, increasingly many censored subjects \textbf{(Center)}, and increasingly large planted subgroups \textbf{(Right)}. \EDS is the closest competitor to \ourmethod closely followed by \RK. Higher is better. The shaded areas show ±1 standard error over 10 runs.}
    \label{fig:synth_panel_bands}
\end{figure*}

%% file: texfigs/real.tex
\begin{wrapfigure}[18]{r}{0.5\textwidth}
    \vspace{-5pt}
    \centering
    \scriptsize
    \subfloat[Unemployment dataset]{%
        \centering
        \begin{tikzpicture}
        \begin{groupplot}[
            group style={
                group size=4 by 1,
                horizontal sep=25pt,
            },
            ytick align=center,
            ylabel style={yshift=5pt},
            xlabel style={yshift=-5pt},
            width=\textwidth/3.85,
            ymin=0,
            ymax=1,
            xmin=0,
        ]
            \nextgroupplot[
            pretty line,
            pretty fill legend,
            xtick={0,7,14,21,28},
            legend style={
                at={(-0.2,1.2)},
                anchor=south west,
                legend columns=4,
                overlay
            },
            cycle list name=ss-rk-pop-colors,
            xlabel={Unemployment time (w)},
            ylabel={Survival prob.},
            legend entries={\ourmethod,, \RK,, \EDS,, Population},
            ]
                \foreach \i in {SySurv, RuleKit} {
                    \addplot+[very thick] table[x=x, y=y] {data/exps/UnempDur_\i_KM.tsv};
                    \addplot+[forget plot, draw=none, name path=low] table[x=x, y=y_c0] {data/exps/UnempDur_\i_KM.tsv};
                    \addplot+[forget plot, draw=none, name path=up] table[x=x, y=y_c1] {data/exps/UnempDur_\i_KM.tsv};
                    \addplot fill between[of=low and up];
                }
            \addplot+[very thick, draw=none] table[x=x, y=y] {data/exps/UnempDur_Population_KM.tsv};
            \addplot+[forget plot, draw=none, name path=low] table[x=x, y=y_c0] {data/exps/UnempDur_Population_KM.tsv};
            \addplot+[forget plot, draw=none, name path=up] table[x=x, y=y_c1] {data/exps/UnempDur_Population_KM.tsv};
            \addplot[opacity=0.0] fill between[of=low and up];

            \addplot+[very thick] table[x=x, y=y] {data/exps/UnempDur_Population_KM.tsv};
            \addplot+[forget plot, draw=none, name path=low] table[x=x, y=y_c0] {data/exps/UnempDur_Population_KM.tsv};
            \addplot+[forget plot, draw=none, name path=up] table[x=x, y=y_c1] {data/exps/UnempDur_Population_KM.tsv};
            \addplot fill between[of=low and up];
            
            \end{groupplot}
        \end{tikzpicture}
        \label{fig:real1}
    }
    \subfloat[Heart attack dataset]{%
        \centering
        \begin{tikzpicture}
        \begin{groupplot}[
            group style={
                group size=4 by 1,
                horizontal sep=25pt,
            },
            ytick align=center,
            ylabel style={yshift=5pt},
            xlabel style={yshift=-5pt},
            width=\textwidth/3.85,
            ymin=0,
            ymax=1,  
            xmin=0,
        ]   

            \nextgroupplot[
            pretty line,
            pretty fill legend,
            ylabel={Survival prob.},
            cycle list name=ss-eds-pop-colors,
            xlabel={Survival time (y)},
            ]
                \foreach \i in {SySurv, EsmamDS, Population} {
                    \addplot+[very thick] table[x=x, y=y] {data/exps/TRACE_\i_KM.tsv};
                    \addplot+[forget plot, draw=none, name path=low] table[x=x, y=y_c0] {data/exps/TRACE_\i_KM.tsv};
                    \addplot+[forget plot, draw=none, name path=up] table[x=x, y=y_c1] {data/exps/TRACE_\i_KM.tsv};
                    \addplot fill between[of=low and up];
                }
            \end{groupplot}
        \end{tikzpicture}
        \label{fig:real2}
    }
    \caption{\emph{Real-world setting.} Survival subgroups discovered in the unemployment (\textbf{a}) and heart attack (\textbf{b}) datasets using \ourmethod and \RK resp. \EDS. \ourmethod learns more exceptional subsets of the subgroups discovered by baselines. The shaded areas show 95\% confidence intervals.}
    \label{fig:real}
\end{wrapfigure}

%% file: texfigs/study.tex
\begin{wrapfigure}[18]{r}{0.5\textwidth}
    \vspace{-10pt}
    \centering
    \scriptsize
    \subfloat[Primary cohort]{%
        \centering
        \begin{tikzpicture}
        \begin{groupplot}[
            group style={
                group size=1 by 1,
                horizontal sep=25pt,
            },
            ytick align=center,
            width=\textwidth/3.85,
            ylabel style={yshift=5pt},
            xlabel style={yshift=-5pt},
            ymin=0,
            ymax=1,  
            xmin=0,
        ]
            \nextgroupplot[
            pretty line,
            pretty fill legend,
            ylabel={Survival prob.},
            cycle list name=case-colors,
            xlabel={Follow up time (mo)},
            legend style={
                at={(1,1.2)},
                anchor=south east,
                legend columns=3,
            },
            legend entries={$S_1(t)$,,$S_2(t)$,,Pop.,,$S_3(t)$,,$S_4(t)$}
            ]
                \foreach \i in {0, 1, Population, 3, 4} {
                    \addplot+[very thick] table[x=x, y=y] {data/exps/mmc6_nCounter_pRCTx_\i_KM.tsv};
                    \addplot+[forget plot,draw=none, name path=low] table[x=x, y=y_c0] {data/exps/mmc6_nCounter_pRCTx_\i_KM.tsv};
                    \addplot+[forget plot,draw=none, name path=up] table[x=x, y=y_c1] {data/exps/mmc6_nCounter_pRCTx_\i_KM.tsv};
                    \addplot fill between[of=low and up];
                }
                
            \end{groupplot}
        \end{tikzpicture}
        \label{fig:study1}
    }
    \subfloat[Postoperative cohort]{%
        \centering
        \begin{tikzpicture}
        \begin{groupplot}[
            group style={
                group size=1 by 1,
                horizontal sep=25pt,
            },
            ytick align=center,
            width=\textwidth/3.85,
            ylabel={Survival prob.},
            ylabel style={yshift=5pt},
            xlabel style={yshift=-5pt},
            xtick={0,25,50,75,100},
            ymin=0,
            ymax=1,  
            xmin=0,
        ]

            \nextgroupplot[
            pretty line,
            pretty fill legend,
            cycle list name=case-colors2,
            xlabel={Follow up time (mo)},
            legend style={
                at={(1,1.2)},
                anchor=south east,
                legend columns=3,
            },
            legend entries={$S_5(t)$,,$S_6(t)$,,Pop.,,$S_7(t)$,,$S_8(t)$}
            ]
                \foreach \i in {0, 1, Population, 3, 4} {
                    \addplot+[very thick] table[x=x, y=y] {data/exps/mmc6_nCounter_PostOp_\i_KM.tsv};
                    \addplot+[forget plot,draw=none, name path=low] table[x=x, y=y_c0] {data/exps/mmc6_nCounter_PostOp_\i_KM.tsv};
                    \addplot+[forget plot,draw=none, name path=up] table[x=x, y=y_c1] {data/exps/mmc6_nCounter_PostOp_\i_KM.tsv};
                    \addplot fill between[of=low and up];
                }
            \end{groupplot}
        \end{tikzpicture}
        \label{fig:study2}
    }
    \caption{\emph{Case study.} \ourmethod discovers, overall, sets of diverse subgroups in the primary cohort (\textbf{a}), and two subgroups of poor responders in the postoperative cohort (\textbf{b}) of HNSCC data. The shaded areas show 95\% confidence intervals.}
    \label{fig:study}
\end{wrapfigure}

%% file: sections/conclusion.tex
\section{Conclusion}

We propose \ourmethod, a method for survival subgroup discovery. \ourmethod leverages individual survival functions obtained through non-parametric survival regression, namely, Random Survival Forests. In particular, we overcome the common drawbacks of existing approaches that rely on parametric assumptions on the underlying survival structure while also obscuring intra-subgroup variability by operating at the level of aggregate statistics. \ourmethod exploits individual-level deviations in survival w.r.t the population in the way it quantifies subgroup exceptionality, achieving a more sensitive measure compared to prior work. \ourmethod employs a rule learner whose parameters can be learned to automatically select features and cutoffs along their domains to select members of a subgroup. By doing so, \ourmethod effectively results in human-interpretable rules that describe subgroups with exceptional survival behavior as it maximizes subgroup exceptionality via gradient-based optimization. Extensive experiments on real-world datasets demonstrate that \ourmethod consistently outperforms existing baselines in terms of subgroup exceptionality.
In a case study on neck cancer patients, \ourmethod finds biomarkers that are known to be associated with good, respectively, poor response to treatment while also identifying novel subgroups that may warrant further investigation.

\textbf{Limitations.} \ourmethod provides a novel method to uncover, possibly novel, subgroups with exceptional survival characteristics. \ourmethod is strictly associational, and while its results may indicate (unknown) causal mechanisms, it does not provide any guarantees in this regard. Before putting any of its results into (clinical) practice, randomized controlled trials, or other ways to verify causality should be employed. Moreover, \ourmethod relies on the population model to obtain individual survival functions, which dictates the quality of the results. In particular, our objective may become overestimated, but there exist asymptotic  methods for debiasing it \citep{ito:2018:unbiased}.

\textbf{Impact.} Nevertheless, given these potentially interesting cohorts, researchers may reproduce previous findings or make new discoveries that allow for more informed decisions about treatment and intervention strategies. We underscore the importance of responsible use of our method in critical settings: the rules \ourmethod learns are based on correlations extracted from potentially selection-biased observational data. Thus, any decisions based on these rules must be made with caution and in conjunction with domain experts as misuse may lead to unfair treatment of specific demographics.

\textbf{Future Work.} In \ourmethod, we consider tabular data; however, many applications involve other data modalities such as images or sequences. We intend to extend \ourmethod to incorporate such structured data alongside tabular covariates in a meaningful way to discover survival subgroups characterized by both types of data. This may involve anchoring the discovered subgroups in interpretable tabular features as well as emergent visual features, or sequential patterns. Furthermore, although we do not constrain \ourmethod to finding compact, low variance, subgroups so as to keep it as sensitive as possible, intra-subgroup cohesion can be a favorable inductive bias to enforce for some applications \citep{boley:2017:identifying}. This can be achieved by adding a tunable regularization term to our objective that encourages low intra-subgroup variance in survival, at a computational cost.

%% file: sections/acknowledgements.tex
Jawad Al Rahwanji is supported by DFG GRK 2853/1 ``Neuroexplicit Models of Language, Vision, and Action''.

%% file: sections/appendix.tex
\section{Proof of Proposition \ref{prop:individual}}
\label{apx:proof}

\Individual*

\begin{proof}
    Let $Z=s(\mathbf{X})-s_B$ be a random variable. The triangle inequality property of the absolute value function (Jensen's inequality applied to the convex absolute value function) states that
    \[
        \mathbb{E}[\;|Z|\;]\geq|\;\mathbb{E}[Z]\;|\;.
    \]
    It also holds to condition the expectations using the indicator $\sigma_A(\mathbf{x})=1$ to get
    \[
        \mathbb{E}_{\mathbf{x}\sim P_{\mathbf{X}}}[\;|Z|\;\mid\sigma_A(\mathbf{x})=1\;]\geq|\;\mathbb{E}_{{\mathbf{x}\sim P_{\mathbf{X}}}}[Z\mid\sigma_A(\mathbf{x})=1]\;|\;.
    \]
    We substitute the definition of $Z$ into the inequality to get
    \[
        \mathbb{E}_{\mathbf{x}\sim P_{\mathbf{X}}}[\;|s(\mathbf{x})-s_B|\;\mid\sigma_A(\mathbf{x})=1]\geq|\;\mathbb{E}_{{\mathbf{x}\sim P_{\mathbf{X}}}}[s(\mathbf{x})-s_B\mid\sigma_A(\mathbf{x})=1]\;|\;.
    \]
    We use the linearity of the expectation, $\mathbb{E}[Z+W]=\mathbb{E}[Z]+\mathbb{E}[W]$, and note that the expectation of the positive constant $s_B$ is the constant itself to get
    \[
        \mathbb{E}_{\mathbf{x}\sim P_{\mathbf{X}}}[\;|s(\mathbf{x})-s_B|\;\mid\sigma_A(\mathbf{x})=1\;]\geq|\;\mathbb{E}_{\mathbf{x}\sim P_{\mathbf{X}}}[s(\mathbf{x})\mid\sigma_A(\mathbf{x})=1]-s_B\;|\;.
    \]
    Lastly, we substitute the definitions of $\ell_t^1(\cdot,\cdot)$ and $s_A$ to arrive at
    \begin{align}
        \mathbb{E}_{\mathbf{x}\sim P_{\mathbf{X}}}[\;\ell_t^1(s(\mathbf{x}),s_B)\mid\sigma_A(\mathbf{x})=1\;]\geq\ell_t^1(s_A,s_B)\;.
    \end{align}
\end{proof}

\begin{corollary}
    \label{cor:individual}
    Under the assumptions of Prop.~\ref{prop:individual}, if there exist $\mathbf{x}^{(1)},\mathbf{x}^{(2)}\in\mathcal{X}$ selectable by $\sigma_A$ with a strictly positive probability measure such that 
    \[
        s(\mathbf{x}^{(1)}) > s_B\land s(\mathbf{x}^{(2)}) < s_B\;,
    \] 
    then the inequality in Prop.~\ref{prop:individual} is strict, i.e.
    \[
        \mathbb{E}_{\mathbf{x}\sim P_{\mathbf{X}}}[\ell^1_t(s(\mathbf{x}),s_B)\mid\sigma_A(\mathbf{x})=1]>\ell^1_t(s_A,s_B)\;.
    \]
\end{corollary}

\begin{proof}
    Proposition \ref{prop:individual} follows from the triangle inequality property of the absolute value function. Equality holds iff all arguments are aligned, i.e.~either $s(\mathbf{x})\geq s_B$ or $s(\mathbf{x})\leq s_B$ for all $\mathbf{x}$ selectable by $\sigma_A$. The existence of $\mathbf{x}^{(1)}$ and $\mathbf{x}^{(2)}$ contradicts this, hence the inequality is strict.
\end{proof}

\newpage

\section{The \ourmethod Algorithm}
\label{apx:algorithm}

In this section, we provide the pseudocode for \ourmethod in Alg.~\ref{alg:learn_sg}. Prior to executing \ourmethod, we fit a global non-parametric population model to the entire dataset $\mathit{RSF}(\mathbf{X}, T, \Delta)$. This gives per-subject survival functions over the discrete domain of unique event times $\hat{S}_D(t^*\mid \mathbf{X}=\mathbf{x}^{(i)})$ as the matrix $M$. The rule set $R$ is initialized as $\emptyset$ and is populated with subsequently discovered subgroups, up to a user-defined limit. On the subject level $\mathbf{x}^{(i)}$, the following steps are applied. First, each feature $x^{(i)}_j$ is binned with learned cut points $\alpha_j$ and $\beta_j$ using the soft condition to obtain $\hat{\pi}(\mathbf{x}^{(i)};\boldsymbol{\alpha},\boldsymbol{\beta},\tau)$. Next, by means of the weights $w_j$, we combine per-feature conditions $\hat{\pi}(x^{(i)}_j)$ into a conjunction $\hat{\sigma}(\mathbf{x}^{(i)};\boldsymbol{\alpha},\boldsymbol{\beta},\mathbf{w},\tau)$. Then, we estimate the expected per-subject deviation in survival within the subgroup w.r.t. the overall population according to our objective. In a backwards pass, the soft rule parameters are updated to select a more exceptional subgroup for the remainder of the epochs. We anneal the temperature $\tau$ once halfway through the epochs and once three quarters of the way. Note that $\#\mathit{events}$ denotes the number of unique event times $t^*$ in $D$, where $t^*\in\textnormal{uniq}(\{t_i\sim P_T\mid\delta_i=1\})$. The \textsc{MeasureExceptionality} algorithm, used in Alg.~\ref{alg:learn_sg}, computes exceptionality using the trapezoidal rule over the absolute differences in survival between each subject and the reference population survival. See Appx.~\ref{apx:complexity} for a discussion on computational complexity and optimizations.

\begin{algorithm}[]
    \caption{\textsc{LearnSubgroup}}
    \label{alg:learn_sg}
    \begin{algorithmic}[1]
        \STATE \textbf{Input:} design matrix $\mathbf{X}$, population model $M$, rules $R$, size penalty $\gamma$, initial temperature $\tau$
        \STATE \textbf{Output:} a rule $\hat{\sigma}$ that selects an exceptional subgroup
        \STATE $\alpha_j\gets\min\;X_j$
        \STATE $\beta_j\gets\max\;X_j$
        \STATE $w_j\gets 1$
        \STATE $\text{Initialize soft rule} \;\hat{\sigma}(\cdot;\boldsymbol{\alpha},\boldsymbol{\beta},\mathbf{w},\tau)$
        \STATE $S_D(t^*)\gets \frac{1}{n}\sum^n_{i=1}S(t^*\mid\mathbf{x}^{(i)})$
        \STATE $\mathit{Exceptionality}_D(\mathbf{x})\gets\textsc{MeasureExceptionality}(\mathbf{x}, M, S_D(t^*))$
        \FOR{$e\gets 1$ \textbf{to} $\;\mathit{\#epochs}$}
            \STATE $\mathit{Membership}(\mathbf{x})\gets \hat{\sigma}(\mathbf{x};\boldsymbol{\alpha},\boldsymbol{\beta},\mathbf{w},\tau)$
            \STATE $\mathit{Size}\gets \sum^n_{i=1}\mathit{Membership}(\mathbf{x}^{(i)})$
            \STATE $\hat{\phi}\gets \sum^n_{i=1}\mathit{Membership}(\mathbf{x}^{(i)})\cdot\mathit{Exceptionality}_D(\mathbf{x^{(i)}})$
            \STATE $\mathit{Weighted\_}\hat{\phi}\gets\hat{\phi}\cdot\mathit{Size}^{\gamma-1}$
            \STATE $\mathit{Regularizer}\gets0$
            \FORALL{$\hat{\sigma}_g\in R$}
                \STATE $\mathit{Size}_{g}\gets \sum^n_{i=1}\hat{\sigma}_g(\mathbf{x}^{(i)})$
                \STATE $S_g(t^*)\gets \frac{1}{\mathit{Size}_g}\sum^n_{i=1}S(t^*\mid\mathbf{x}^{(i)})\cdot \hat{\sigma}_g(\mathbf{x}^{(i)})$
                \STATE $\mathit{Exceptionality}_g(\mathbf{x})\gets\textsc{MeasureExceptionality}(\mathbf{x}, M, S_g(t^*))$
                \STATE $\hat{\phi}_g\gets\sum^n_{i=1}\mathit{Membership}(\mathbf{x}^{(i)})\cdot\mathit{Exceptionality}_g(\mathbf{x^{(i)}})$
                \STATE $\mathit{Regularizer}\gets \mathit{Regularizer}+\hat{\phi}_g$
            \ENDFOR
            \STATE $\mathit{Regularizer}\gets\mathit{Regularizer}\cdot\mathit{Size}^{\gamma-1}$
            \STATE $\mathit{loss}\gets-\mathit{Weighted\_}\hat{\phi}-\mathit{Regularizer}$
            \STATE $\text{Update the parameters of rule $\hat{\sigma}$ to minimize the loss}$
            \IF{$(e=\mathit{\#epochs}/2)\lor(e=\mathit{\#epochs}\cdot3/4)$}
                \STATE $\tau\gets \tau/2$
            \ENDIF
        \ENDFOR
        \STATE \textbf{return} $\;\hat{\sigma}$
    \end{algorithmic}
\end{algorithm}

\begin{algorithm}[]
    \caption{\textsc{MeasureExceptionality}}
    \label{alg:int_abs_diff}
    \begin{algorithmic}[1]
        \STATE \textbf{Input:} subject covariates $\mathbf{x}$, population model $M$, reference survival at the event times $S_\mathit{ref}(t^*)$
        \STATE \textbf{Output:} subject survival exceptionality value w.r.t reference
        \STATE $\mathit{abs\_diff}\gets|\;S(t^*\mid\mathbf{x})-S_\mathit{ref}(t^*)\;|$
        \STATE \textbf{return} $\;\sum_{u=0}^{\#\mathit{events}-1}(t^*_{u+1}-t^*_u)/2\cdot\left(\mathit{abs\_diff}[u]+\mathit{abs\_diff}[u+1]\right)$
    \end{algorithmic}
\end{algorithm}

\newpage

\section{Rule Pruning}
\label{apx:pruning}

We provide the pseudocode for our post hoc pruning algorithm that simplifies rules in the presence of collinearity in Alg.~\ref{alg:prune}. The \textsc{PruneRule} algorithm is a greedy procedure for reducing the complexity of learned rules without significantly altering its subject memberships. It begins with a full rule and iteratively identifies the least influential condition and attempts to remove it by setting its weight to zero. The algorithm then evaluates each candidate rule by computing the Jaccard similarity between its membership indicators and those of the full rule. This process continues as long as the similarity remains above a specified threshold, ensuring it still describes, more or less, the same subset of the population. It is to (optionally) be executed after learning a rule using Alg.~\ref{alg:learn_sg}.

\begin{algorithm}[H]
    \caption{\textsc{PruneRule}}
    \label{alg:prune}
    \begin{algorithmic}[1]
        \STATE \textbf{Input:} design matrix $\mathbf{X}$, learned rule $\hat{\sigma}$, Jaccard similarity $\mathit{threshold}$ (default: 0.95)
        \STATE \textbf{Output:} pruned rule $\hat{\sigma}_{\mathit{prun}}$

        \STATE $\hat{\sigma}_{\mathit{prun}} \gets \hat{\sigma}$
        \STATE $\mathit{mask}\gets\hat{\sigma}(\mathbf{X})=1$

        \WHILE{$\mathit{Jaccard}(\mathit{mask}, \hat{\sigma}_\mathit{prun}(\mathbf{X})=1) \ge \mathit{threshold}$}
            \STATE $\mathit{indices} \gets \{\;j\mid \mathit{ReLU}(w_j^{\hat{\sigma}_{\mathit{prun}}}) > 0.1\;\}$

            \STATE $\mathit{score^\star} \gets 0$
            \STATE $\mathit{idx^\star} \gets \mathit{null}$

            \FOR{each $j \in \mathit{indices}$}
                \STATE $\hat{\sigma}^\prime \gets \hat{\sigma}_{\mathit{prun}}$
                \STATE $w_j^{\hat{\sigma}^\prime} \gets 0$
                \STATE $\mathit{score} \gets \mathit{Jaccard}(\mathit{mask},\hat{\sigma}^\prime(\mathbf{X})=1)$

                \IF{$\mathit{score} > \mathit{score^\star}$}
                    \STATE $\mathit{score^\star} \gets \mathit{score}$
                    \STATE $\mathit{idx^\star} \gets j$
                \ENDIF
            \ENDFOR

            \IF{$\mathit{score^\star} < \mathit{threshold}$}
                \STATE \textbf{break}
            \ENDIF

            \STATE $w_{\mathit{idx^\star}}^{\hat{\sigma}_{\mathit{prun}}} \gets 0$
        \ENDWHILE
        \STATE \textbf{return} $\hat{\sigma}_{\mathit{prun}}$
    \end{algorithmic}
\end{algorithm}

\newpage

\section{Statistical Validation}
\label{apx:validation}

We provide the pseudocode of our instantiation for building the model of false discoveries (DFD), proposed by~\citet{duivesteijn:2011:exploiting}, in Alg.~\ref{alg:build_dfd}. The \textsc{BuildDFD} algorithm quantifies the exceptionality that can be expected by chance. By repeatedly shuffling the survival outcomes relative to the design matrix, the algorithm creates a series of null datasets where no true relationship exists. In each iteration, it runs \ourmethod while recording the exceptionality scores (line 12 of Alg.~\ref{alg:learn_sg}, divided by the subgroup size from line 11) obtained from these null datasets. It returns the mean and standard deviation that parametrize a normal distribution of false exceptionalities. This distribution serves as a baseline (null model) for judging the significance of discovered subgroups in terms of their exceptionalities, subject to multiple hypothesis testing. For that, we use the Z-test to obtain one-tailed p-values with a nominal alpha cutoff of 0.05. \textsc{BuildDFD} is to be executed before running Alg.~\ref{alg:learn_sg}. \emph{$\#runs$ should be no less than 1000.}

In Fig.~\ref{fig:dfd}, we show an example of the distribution of false discoveries obtained from 1000 runs of \textsc{BuildDFD} on the postoperative cohort from our case study, with a fitted normal distribution. The mean and standard deviation of the fitted normal distribution are $\mu=16.9$ and $\eta=2.77$, respectively.

\begin{algorithm}[H]
    \caption{\textsc{BuildDFD}}
    \label{alg:build_dfd}
    \begin{algorithmic}[1]
        \STATE \textbf{Input:} design matrix $\mathbf{X}$, time-to-event $T$, event indicator $\Delta$, size penalty $\gamma$, initial temperature $\tau$, number of independent samples $\mathit{\#runs}$ (default: 1000)
        \STATE \textbf{Output:} parameters of the Gaussian representing the distribution of false discoveries, mean $\mu$, and standard deviation $\eta$

        \STATE $\mathit{scores} \gets \emptyset$
        \STATE $\mathit{indices} \gets \{1, \dots, n\}$
        \STATE $\mathit{permutations} \gets \{\;P_r \mid P_r \text{ is a random permutation of } \mathit{indices}, r=1\dots \#runs\;\}\;$

        \FOR{each $P_r \in \mathbf{P}$}
            \STATE $T^\prime \gets T[P_r]$
            \STATE $\Delta^\prime \gets \Delta[P_r]$

            \STATE $M\gets\mathit{RSF}(\mathbf{X},T^\prime,\Delta^\prime)$

            \STATE $\hat{\sigma} \gets \textsc{LearnSubgroup}(\mathbf{X}, M, \emptyset, \gamma, \tau)$

            \STATE $\mathit{scores} \gets scores\cup\{\hat{\phi}\cdot\mathit{Size}^{-1}\mid\hat{\phi} \text{ is the exceptionality of the rule }\hat{\sigma}\}$
        \ENDFOR

        \STATE $\mu \gets \text{mean}(\mathit{scores})$
        \STATE $\eta \gets \text{std}(\mathit{scores})$

        \STATE \textbf{return} $\mu,\eta$
    \end{algorithmic}
\end{algorithm}

\begin{figure}[b]
    \centering
    \begin{tikzpicture}
        \begin{axis}[
            pretty line,
            pretty grid line,
            width=\textwidth/2,
            xlabel={Score},
            ytick={0,0.05,0.10,0.15},
            yticklabel style={/pgf/number format/fixed},
            xtick={0,5,10,15,20,25,30},
            ylabel style={yshift=10pt},
            xlabel style={yshift=-10pt},
            ymin=0,
            xmin=0,
            ylabel={Probability density},
            legend style={
                at={(1.3,0.9)},
                anchor=north east,
                legend columns=1,
            }
        ]

            \addplot [
                pretty ybar interval,
                fill=pr-color1a!50,
                draw=pr-color1a!80,
                opacity=0.6,
                hist={
                    bins=30,
                    density=true
                }
            ] table [col sep=comma, y=score] {data/exps/mmc6_nCounter_PostOp_0.1_dfd_scores.csv};
            \addlegendentry{Sample data}

            \def\MU{16.9}
            \def\SIGMA{2.77}

            \addplot [
                very thick, draw=pr-color1b,
                domain=\MU-6*\SIGMA:\MU+6*\SIGMA,
                samples=100
            ] {exp(-(x-\MU)^2 / (2*\SIGMA^2)) / (\SIGMA * sqrt(2*pi))};
            \addlegendentry{$\mathcal{N}(\MU, \SIGMA)$}

            \draw [dashed, thick, draw=pr-color1b] (axis cs:\MU, \pgfkeysvalueof{/pgfplots/ymin}) -- (axis cs:\MU, \pgfkeysvalueof{/pgfplots/ymax});

        \end{axis}
    \end{tikzpicture}
    \caption{Histogram of the exceptionality scores obtained from 1000 runs of \textsc{BuildDFD} on the postoperative cohort from our case study, with a fitted normal distribution. The mean and standard deviation of the fitted normal distribution are $\mu=16.9$ and $\eta=2.77$, respectively.}
    \label{fig:dfd}
\end{figure}

\newpage

\section{Hyperparameters of \ourmethod and Baselines}
\label{apx:hyperparams}

In this section, we discuss the hyperparameter choices for \ourmethod and the baselines used in our experiments.

\subsection*{\ourmethod}

We have two core hyperparameters ($\gamma$ and $\tau$) in \ourmethod, and some more dictated by the choice of the regression model. The most important parameter is $\gamma$. Tuning $\gamma$ is crucial for steering the learning towards rules that select subgroups of a desired size range. Values of $\gamma$ range from 0.0 (very small subgroup) to 1.0 (a very large subgroup) at 0.1 increments. For our experiments, we optimize $\gamma$ and found that $\gamma\in[0,0.2]$ was almost always selected. As for $\tau$, we found that $\tau=0.2$ worked well across settings and datasets. We have secondary parameters, namely, the number of epochs and the learning rate used for stochastic optimization. Generally, 1000 epochs paired with a learning rate of 0.01 ensured stable convergence. Since we always start with a full rule, i.e.~a rule that selects all subjects by spanning the entire feature space, initialization does not affect the optimization, and because we anneal the temperature, this ensures convergence to a (local) minimum.

In this work, we instantiated \ourmethod using Random Survival Forests, but in theory it can be done using any non-parametric continuous-time regression model. We do not sample features for each tree, despite the computational cost, to avoid missing important features. However, we do sample subjects, especially when $n$ is very large. To train the RSF, we used 100 trees (300 for the case study, to account for its rather small dataset sizes), a maximum depth equal to double the number of features, and a maximum of 2000 subjects per tree. A rule of thumb for getting a reliable splitting is to have at least 10--20 events in a cohort. As such, we set the minimum number of subjects to split to 40 and, correspondingly, the minimum number of subjects in a leaf to 20.

\subsection*{Baselines}

We use the implementations of the baselines as provided by their authors. We take the top-5 subgroups discovered by each baseline for evaluation. We preprocess the data in the same way for each of the baselines. This only involves the use of a $k$-bin discretizer for continuous features using $k=5$ and the \emph{quantile} strategy.

\section{Computational Complexity of \ourmethod}
\label{apx:complexity}

Modern implementations of RSFs are efficient. Nevertheless, fitting the population model requires the most time. During training, we have a computational complexity of $\mathcal{O}(\mathit{\#trees}\cdot N\cdot V\cdot d/\mathit{\#cores})$. $\#trees$ is fixed, $N$ is the number of subjects assigned to a tree, $V$ is the number of features considered for splitting, $d$ is the tree depth, and $\mathit{\#cores}$ is also fixed. After training, the computational complexity of predicting on the entire dataset is $\mathcal{O}(\mathit{\#trees}\cdot N \cdot d/\mathit{\#cores})$, which happens only once, and the results are saved.

During learning, the computational complexity is $\mathcal{O}(\mathit{\#epochs}\cdot N\cdot V)$, where $\mathit{\#epochs}$ is fixed. We omit this in Alg.~\ref{alg:learn_sg}, but we save unchanging data structures for efficiency. All the computations are done using highly optimized matrix operations through specialized libraries.

\textbf{Hardware.} For our experiments, we used machines equipped with 2x AMD Epyc 7773x 2.2GHz (base), 3.5GHz (max boost) with 128 real cores and 2TB of memory.

\newpage

\section{Survival Data Generation}
\label{apx:data}

In Alg.~\ref{alg:make_sd}, we provide the pseudocode for, \textsc{MakeSurvivalData}, the routine we use for generating synthetic survival data. \textsc{MakeSurvivalData} embeds a hidden subgroup within a larger population. It selects $k$ features to define a hyper-box, or \emph{target}, region in the feature space and partitions subjects based on whether their features fall within these bounds. It assigns subjects different Weibull-distributed event times based on their subgroup status, while also incorporating a linear covariate influence for variability. It also validates the resulting scales and adjusts them such that they are always positive and do not overlap between the subgroup and the population. Finally, it applies a censoring mechanism to a percentage of the subjects, where censored subjects have their event times drawn uniformly between 0 and their true event time.

\begin{algorithm}[H]
    \caption{\textsc{MakeSurvivalData}}
    \label{alg:make_sd}
    \begin{algorithmic}[1]
        \STATE \textbf{Input:} number of samples $n$, number of covariates $p$, number of true conditions $k$, Weibull scale of non-subgroup subjects $\mathit{scale_\mathit{nsg}}$, Weibull shape of non-subgroup subjects $\mathit{shape_\mathit{nsg}}$, Weibull scale of subgroup subjects $\mathit{scale_\mathit{sg}}$, Weibull shape of subgroup subjects $\mathit{shape_\mathit{sg}}$, percentage of subjects in the subgroup $\mathit{ratio_\mathit{target}}$, percentage of censored subject outcomes $\mathit{ratio_\mathit{cens}}$
        \STATE \textbf{Output:} synthetic design matrix $\mathbf{X}$, synthetic time-to-event $T$, synthetic event indicator $\Delta$
        \STATE $\mathcal{V} \gets \mathit{Sample}(p, k, \text{without replacement})$
        \STATE $\epsilon \leftarrow \sqrt[k]{\mathit{ratio_\mathit{target}}}$
        \FOR{$j \gets 1$ \textbf{to} $\;k$}
            \STATE $\alpha_j \sim \mathit{Uniform}(0,1-\epsilon)$
        \ENDFOR
        \STATE $\mathbf{X} \sim \mathit{Uniform}(0, 1)^{n \times p}$
        \FOR{$j \gets 1$ \textbf{to} $\;k$}
            \STATE $\mathbf{X}_\mathit{sg}[:, j] \sim \mathit{Uniform}(\alpha_j, \alpha_j+\epsilon)^{n}$
            \STATE $\mathbf{X}_\mathit{nsg}[:, j] \sim \mathit{Uniform}(0, 1)^{n}$
            \WHILE{$\exists\;\mathbf{x}_\mathit{nsg}[j]\;\mid\;\alpha_j\leq\mathbf{x}_\mathit{nsg}[j]\leq\alpha_j+\epsilon$}
                \STATE $\mathbf{X}_\mathit{nsg}[:, j] \sim \mathit{Uniform}(0, 1)^{n}$
            \ENDWHILE
        \ENDFOR
        \STATE $\sigma \sim \mathit{Bernoulli}(\mathit{ratio_\mathit{target}})^{n}$
        \FOR{$j \gets 1$ \textbf{to} $\;k$}
            \STATE $\mathbf{X}[\sigma=1, v_j] \gets \mathbf{X}_\mathit{sg}[\sigma=1,j]$
            \STATE $\mathbf{X}[\sigma=0, v_j] \gets \mathbf{X}_\mathit{nsg}[\sigma=0,j]$
        \ENDFOR
        \STATE $\mathcal{N}_0 \gets \sum_{j=1}^k \mathbf{X}[:, v_j] - k/2$
        \STATE $\psi\gets 1$
        \STATE $\mathit{scale_\mathit{sg}} \gets \mathit{scale_\mathit{sg}} + \mathcal{N}_0 \cdot \psi$
        \STATE $\mathit{scale_\mathit{nsg}} \gets \mathit{scale}_{\mathit{nsg}} + \mathcal{N}_0 \cdot \psi$
        \WHILE{$\mathit{min}(\mathit{scale_\mathit{sg}})<0\lor\mathit{min}(\mathit{scale_\mathit{nsg}})<0\lor|\mathit{Supp}(\mathit{scale_\mathit{{nsg}}})\cap\mathit{Supp}(\mathit{scale_\mathit{sg}})|>0$}
            \STATE $\psi\gets\psi\cdot0.9$
        \ENDWHILE
        \STATE $Y_{\mathit{sg}} \leftarrow \mathit{scale_\mathit{sg}} \cdot \weibull(\mathit{shape_\mathit{sg}},1)$
        \STATE $Y_{\mathit{nsg}} \leftarrow \mathit{scale_\mathit{nsg}} \cdot \weibull(\mathit{shape_\mathit{nsg}},1)$
        \STATE $Y[\sigma=1] \leftarrow Y_{\mathit{sg}}[\sigma=1]$
        \STATE $Y[\sigma=0] \leftarrow Y_{\mathit{nsg}}[\sigma=0]$
        \STATE $\Delta \sim \mathit{Bernoulli}(1 - \mathit{ratio_\mathit{cens}})$
        \STATE $T[\Delta = 0] \sim \mathit{Uniform}(0, Y[\Delta = 0])$
        \STATE $T[\Delta = 1] \gets Y[\Delta = 1]$
        \STATE \textbf{return} $\mathbf{X}$, $T$, $\Delta$
    \end{algorithmic}
\end{algorithm}

\newpage

\section{Exceptionality Metrics}
\label{apx:exceptionality}

To evaluate the discovered subgroups on real-world data across methods, we use our (discrete-time) exceptionality measure ($\Phi$), the two-sample logrank statistic ($\mathit{LR}$), and the absolute mean-shift ($\mathit{AMS}$), between the subgroup $Q$ and the population $D$. Those are formally defined as
\begin{align}
    \Phi(Q, D) &= \frac{1}{|Q|}\sum_{i\in Q}\;\sum_{\textit{event times}}\left|\hat{S}(t\mid\mathbf{x}^{(i)}) - \hat{S}_D(t)\right| \\
    AMS(Q, D) &= \left|\left(\frac{1}{|Q|}\sum_{i\in Q}t^{(i)}\right) - \left(\frac{1}{|D|}\sum_{i\in D}t^{(i)}\right)\right| \\
    LR(Q, D) &= \frac{\sum_{\textit{event times}}(d^Q_t-e^Q_t)}{\sqrt{\sum_{\textit{event times}}e^Q_t\cdot(\frac{r_t-d_t}{r_t})\cdot(\frac{r_t-r^Q_t}{r_t-1})}}\;,
\end{align}
where at-risks $r^\circ_t=\left|\{i\in \circ\mid (t^{(i)} > t\land\delta^{(i)}=1)\lor\delta^{(i)}=0\}\right|$, had-events $d^\circ_t=\left|\{i\in \circ\mid t^{(i)}=t\land \delta^{(i)}=1\}\right|$, total at-risks $r_t=r^Q_t+r^D_t$, total had-events $d_t=d^Q_t+d^D_t$, and expected-events $e^\circ_t=d_t\cdot\frac{r^\circ_t}{r_t}$. We use a standard library to compute $\mathit{LR}$.

\section{Connection Between Our Measure and the Logrank}
\label{apx:link2logrank}

We can trace a transition from the continuous-time one-sample logrank to our proposed integral measure in Eq.~\ref{eq:group_abs_dev}. The one-sample logrank essentially compares the observed number of events $d$ to the expected number of events $e$ at the group level across all event times
\begin{align*}
    \mathit{Logrank}&\propto\sum_\textit{event times}\left(d^\mathit{grp}_t-e^\mathit{grp}_t\right) \\
    &=\sum_\textit{event times}r_t^\mathit{grp}\left(\frac{d^\mathit{grp}_t}{r^\mathit{grp}_t}-\frac{d^\mathit{ref}_t}{r^\mathit{ref}_t}\right) \\
    &=\sum_\textit{event times}r_t^\mathit{grp}\left(\hat{\lambda}_\mathit{grp}(t)-\hat{\lambda}_\mathit{ref}(t)\right) \\
    &=\sum_\textit{event times}r_t^\mathit{grp}\Delta\left(\hat{\Lambda}_\mathit{grp}(t)-\hat{\Lambda}_\mathit{ref}(t)\right)\;,
\end{align*}
where $e^\mathit{grp}_t=d_t^\mathit{ref}\cdot r_t^\mathit{grp}/r_t^\mathit{ref}$, $r$ refers to the number of individuals at risk, “$\mathit{grp}$” and “$\mathit{ref}$” refer to the group and reference populations, respectively, $\hat{\lambda}_\circ(t)$ is the hazard function, $\hat{\Lambda}_\circ(t)$ is the cumulative hazard function, and $\Delta(\cdot)$ denotes the discrete step-wise increment.

\paragraph{From Statistic to Distance} The full version divides this score (the numerator) by the pooled standard error to yield a symmetric test statistic. However, because variance is inversely proportional to sample size, this division dynamically scales the output, converting it from a spatial distance into a statistic. To build a valid $L^1$ geometric metric, we must isolate the spatial numerator. Letting $\rho_A(t)$ be the continuous at-risk process for group $A$, this numerator is natively written as a continuous Stieltjes integral over the study duration $\mathcal{T} = [0, t_{\max}]$,
\[
    \mathit{Logrank} \propto \int_{\mathcal{T}} \rho_A(t) \, d\big(\hat{\Lambda}_A(t) - \hat{\Lambda}_B(t)\big) \;.
\]
While isolating the numerator preserves geometric magnitude, it exposes four restrictive characteristics, which we sequentially dismantle to arrive at our integral-based distance metric.

\begin{restatable}{proposition}{ChangeToState}
    \label{prop:state_vs_differential}
    Let $\Sigma_\Lambda(t) = \hat{\Lambda}_A(t) - \hat{\Lambda}_B(t)$ be the continuous cumulative hazard state difference over the time domain $\mathcal{T} = [0, t_{\max}]$. Integrating the state directly over the domain is equivalent to integrating its differentials scaled by a temporal persistence weight, such that
    \[
        \int_{\mathcal{T}} \Sigma_\Lambda(t) \, dt = t_{\max} \Sigma_\Lambda(0) + \int_{\mathcal{T}} (t_{\max} - t) \, d\Sigma_\Lambda(t) \;.
    \]
\end{restatable}

We provide the proof at the end of this section.

\paragraph{From Changes to States} The logrank integrates the hazard differentials $d(\cdot)$, overlooking prior hazard increments. By Prop.~\ref{prop:state_vs_differential}, evaluating the integral of the accumulated \textit{state} directly is better. It inherently applies a temporal persistence weight $(t_{\max} - t)$ to past deviation and naturally captures any initial baseline gap ($t_{\max} \Sigma_\Lambda(0)$), circumventing the logrank's memoryless aggregation.

\begin{restatable}{proposition}{HazardToSurvival}
    \label{prop:hazard_survival_mvt}
    Given two groups $A$ and $B$ for which the expected group-level survival probabilities at any fixed time $t$ are $\hat{S}_A(t)$ and $\hat{S}_B(t)$, respectively, and for which $\hat{\Lambda}_A(t)$ and $\hat{\Lambda}_B(t)$ are their corresponding cumulative hazard functions. The difference in their cumulative hazards is proportional to the difference in survival $\hat{S}_A(t) - \hat{S}_B(t)$, such that
    \[
        \hat{\Lambda}_A(t) - \hat{\Lambda}_B(t) = -\frac{\hat{S}_A(t) - \hat{S}_B(t)}{\hat{S}^*(t)} \;,
    \]
    where $\hat{S}^*(t)$ is exactly the logarithmic mean of the two survival probabilities, formally defined to account for potential curve crossing as
    \[
        \hat{S}^*(t) = 
        \begin{cases} 
            \frac{\hat{S}_A(t) - \hat{S}_B(t)}{\ln(\hat{S}_A(t)) - \ln(\hat{S}_B(t))} & \text{if } \hat{S}_A(t) \neq \hat{S}_B(t) \\
            \hat{S}_A(t) & \text{if } \hat{S}_A(t) = \hat{S}_B(t) 
        \end{cases} \;,
    \]
    which is strictly bounded by their geometric and arithmetic means, $\sqrt{\hat{S}_A(t) \hat{S}_B(t)} < \hat{S}^*(t) < \frac{\hat{S}_A(t) + \hat{S}_B(t)}{2}$ (for $\hat{S}_A(t) \neq \hat{S}_B(t)$).
\end{restatable}

We provide the proof at the end of this section.
    
\paragraph{From Hazards to Survival} If we apply the logrank's implicit continuous risk weighting $\rho_A(t) \approx N_A \hat{S}_A(t)$, while ignoring the effect of censoring, to the accumulated hazard state $\Sigma_\Lambda(t)$ adopted in the previous step, Prop.~\ref{prop:hazard_survival_mvt} reveals a geometric scaling behavior
\[
    \rho_A(t) \Sigma_\Lambda(t) = -\rho_A(t) \frac{\Sigma_S(t)}{\hat{S}^*(t)} \approx -N_A \underbrace{\frac{\hat{S}_A(t)}{\hat{S}^*(t)}}_{\text{Distortion}} \Sigma_S(t) \;. 
\]
The raw spatial distance $\Sigma_S(t)$ is directly multiplied by the ratio $\hat{S}_A(t) / \hat{S}^*(t)$. Because $\hat{S}^*(t)$ is tightly bounded between the two survival curves, this ratio acts as a distortion strictly constrained by $\min(1, \hat{S}_A(t)/\hat{S}_B(t)) < \hat{S}_A(t)/\hat{S}^*(t) < \max(1, \hat{S}_A(t)/\hat{S}_B(t))$, which becomes unbounded as $\hat{S}_B(t)$ approaches zero. Crucially, this multiplier is asymmetric, i.e.~weighting by $\hat{S}_A(t) / \hat{S}^*(t)$ evaluates the geometric deviation differently than if the reference groups were swapped to weight by $\hat{S}_B(t) / \hat{S}^*(t)$.

The full logrank test relies on its denominator to force symmetry back into the numerator. Because we discarded the denominator to preserve spatial geometry, we are left with this exposed, asymmetric distortion. Because a rigorous distance metric must satisfy spatial symmetry ($d(A,B) = d(B,A)$), we strip this proportional weight. Doing so geometrically resolves the numerator's asymmetry and isolates the pure spatial distance $\Sigma_S(t) = \hat{S}_A(t) - \hat{S}_B(t)$.
    
\paragraph{From Signed to Absolute} The logrank relies on the Cox proportional hazards condition for power, which states that the rates at which subgroup and reference survival decay (i.e.~hazards) are constantly proportional over time. This assumption is often violated in practice, for example, when the survival functions for groups of interest cross one another. The logrank test inherently integrates a signed difference. If the survival curves cross, their relative hazard rates invert, causing the logrank's native differentials to flip signs over the domain $\mathcal{T}$. Consequently, the logrank inherently accumulates positive and negative areas that cancel each other out, severely reducing statistical power under crossing hazards. To resolve this, we apply an absolute value to the isolated state, yielding a robust spatial distance measure $|\Sigma_S(t)| = |\hat{S}_A(t) - \hat{S}_B(t)|$. Evaluating this absolute state measures true net geometric distance without over-penalizing high-frequency estimation noise, recovering exactly $\ell^1_t(s_A, s_B)$ from the right-hand side of the inequality in Prop.~\ref{prop:individual}.

\paragraph{From Group to Individuals} Despite the temporal and spatial modifications in Steps 1 through 3, the measure thus far still operates entirely on expected group-level estimates. By definition, the group survival curve is simply the expectation of individual survival probabilities over all subjects selectable by the subgroup rule $\sigma_A$: $\hat{S}_A(t) = \mathbb{E}_{\mathbf{x}\sim P_\mathbf{X}}[\hat{S}(t\mid\mathbf{x}) \mid \sigma_A(\mathbf{x})=1]$. Evaluating the deviation at this group level inherently averages away and obscures individual heterogeneity. We resolve this by taking the expectation \textit{outside} the distance function. 

Crucially, this final step strictly depends on the absolute value applied in Step 3. If we had retained the logrank's signed difference, moving the expectation outside would have no effect due to the linearity of expectation
\[
    \mathbb{E}_{\mathbf{x}\sim P_\mathbf{X}}\big[\hat{S}(t\mid\mathbf{x}) - \hat{S}_B(t) \;\big|\; \sigma_A(\mathbf{x})=1\big] = \hat{S}_A(t) - \hat{S}_B(t) \;.
\]
However, because the absolute value function is convex, applying it \textit{before} the expectation actively captures the internal heterogeneity that the group average cancels out. By Jensen's inequality we get the left-hand side of the inequality in Prop.~\ref{prop:individual}.

Applying these four modifications in sequence, namely, integrating the accumulated state rather than differentials, mapping to survival while removing the logrank numerator's native asymmetric distortion, applying the absolute value function to prevent native cancellation, and moving the expectation outside the absolute difference to capture individual heterogeneity, exactly yields our proposed exceptionality measure from Eq.~\ref{eq:group_abs_dev}
\[
    \phi(\sigma_A, \sigma_B)=\mathbb{E}_{\mathbf{x}\sim P_\mathbf{X}}\left[\ell^1_\mathcal{T}(\hat{S}(t\mid\mathbf{x}), \hat{S}_B(t)) \;\middle|\; \sigma_A(\mathbf{x})=1 \right] \;,
\]
where $\sigma_A$ simplifies to $\sigma$, and $B$ stands for the reference population/dataset $D$.

\subsection*{Proofs of Propositions~\ref{prop:state_vs_differential} and~\ref{prop:hazard_survival_mvt}}
\label{sec:supporting_props}

\ChangeToState*

\begin{proof}
    We apply the standard rule of continuous integration by parts, $\int u \, dv = uv - \int v \, du$, to the integral of the state over $[0, t_{\max}]$. Let $u = \Sigma_\Lambda(t)$ and $dv = dt$. We select the valid antiderivative $v = (t - t_{\max})$. Therefore, $du = d\Sigma_\Lambda(t)$. Substituting these into the integration by parts formula yields
    \[
        \int_0^{t_{\max}} \Sigma_\Lambda(t) \, dt = \Big[ \Sigma_\Lambda(t) (t - t_{\max}) \Big]_0^{t_{\max}} - \int_0^{t_{\max}} (t - t_{\max}) \, d\Sigma_\Lambda(t) \;.
    \]
    Evaluating the boundary conditions for the first term
    \[
        \Big( \Sigma_\Lambda(t_{\max}) \cdot 0 \Big) - \Big( \Sigma_\Lambda(0) \cdot (-t_{\max}) \Big) = t_{\max} \Sigma_\Lambda(0) \;.
    \]
    Substituting this back and absorbing the negative sign into the integral gives the stated equivalence
    \[
        \int_0^{t_{\max}} \Sigma_\Lambda(t) \, dt = t_{\max} \Sigma_\Lambda(0) + \int_0^{t_{\max}} (t_{\max} - t) \, d\Sigma_\Lambda(t) \;.
    \]
\end{proof}

\HazardToSurvival* 

\begin{proof}
    In survival analysis, the (group-level) survival function is related to its corresponding cumulative hazard function thus
    \[
        \hat{S}_\circ(t) = e^{-\hat{\Lambda}_\circ(t)} \implies \hat{\Lambda}_\circ(t) = -\ln(\hat{S}_\circ(t)) \;.
    \]
    Consider the real-valued function $f(v) = e^{-v}$. This function is continuous and differentiable everywhere on $\mathbb{R}$, with its first derivative being $f'(v) = -e^{-v}$. We evaluate the difference in survival based on the cumulative hazards.
    
    \textit{Case 1: The survival functions do not cross ($\hat{S}_A(t) \neq \hat{S}_B(t)$).}
    Since $\hat{S}_A(t) \neq \hat{S}_B(t)$, it follows that $\hat{\Lambda}_A(t) \neq \hat{\Lambda}_B(t)$. We apply the Mean Value Theorem. There exists an intermediate cumulative hazard value, $\hat{\Lambda}^*(t)$, strictly between $\hat{\Lambda}_A(t)$ and $\hat{\Lambda}_B(t)$ such that
    \[
        f(\hat{\Lambda}_A(t)) - f(\hat{\Lambda}_B(t)) = f'(\hat{\Lambda}^*(t))(\hat{\Lambda}_A(t) - \hat{\Lambda}_B(t)) \;.
    \]
    Evaluating $f$ and $f'$ yields
    \[ 
        e^{-\hat{\Lambda}_A(t)} - e^{-\hat{\Lambda}_B(t)} = -e^{-\hat{\Lambda}^*(t)} (\hat{\Lambda}_A(t) - \hat{\Lambda}_B(t)) \;.
    \]
    Substituting the group-level survivals back into the left side gives
    \[
        \hat{S}_A(t) - \hat{S}_B(t) = -e^{-\hat{\Lambda}^*(t)} (\hat{\Lambda}_A(t) - \hat{\Lambda}_B(t)) \;.
    \]
    Let us define the intermediate term $\hat{S}^*(t) = e^{-\hat{\Lambda}^*(t)}$. Because $f(v) = e^{-v}$ is strictly monotonically decreasing, $\hat{S}^*(t)$ is an intermediate survival probability strictly between $\hat{S}_A(t)$ and $\hat{S}_B(t)$. Rearranging the equation to solve for $\hat{S}^*(t)$ and substituting $\hat{\Lambda}_\circ(t) = -\ln(\hat{S}_\circ(t))$ reveals its exact analytical form
    \[
        \hat{S}^*(t) = -\frac{\hat{S}_A(t) - \hat{S}_B(t)}{\hat{\Lambda}_A(t) - \hat{\Lambda}_B(t)} = -\frac{\hat{S}_A(t) - \hat{S}_B(t)}{-\ln(\hat{S}_A(t)) - (-\ln(\hat{S}_B(t)))} = \frac{\hat{S}_A(t) - \hat{S}_B(t)}{\ln(\hat{S}_A(t)) - \ln(\hat{S}_B(t))} \;.
    \]
    This confirms $\hat{S}^*(t)$ is exactly the logarithmic mean. By standard properties of the logarithmic mean for any two distinct positive real numbers, $\hat{S}^*(t)$ is strictly bounded below by the geometric mean $\sqrt{\hat{S}_A(t) \hat{S}_B(t)}$ and above by the arithmetic mean $\frac{\hat{S}_A(t) + \hat{S}_B(t)}{2}$, fulfilling the proposition's bounds. Finally, rearranging the equation to isolate the difference in cumulative hazards yields
    \[
        \hat{\Lambda}_A(t) - \hat{\Lambda}_B(t) = -\frac{\hat{S}_A(t) - \hat{S}_B(t)}{\hat{S}^*(t)} \;.
    \]

    \textit{Case 2: The survival functions cross ($\hat{S}_A(t) = \hat{S}_B(t)$).} 
    At the exact point of crossing, the cumulative hazards must also be equal ($\hat{\Lambda}_A(t) = \hat{\Lambda}_B(t)$), meaning the left side of our proposition evaluates to $0$. However, evaluating the standard logarithmic mean on the right side yields an undefined $\frac{0}{0}$ singularity. To maintain continuity in the geometric space, we evaluate the limit of the logarithmic mean as $\hat{S}_B(t)$ approaches $\hat{S}_A(t)$. Applying L'H\^{o}pital's rule with respect to $\hat{S}_B(t)$ yields
    \[
        \lim_{\hat{S}_B(t) \to \hat{S}_A(t)} \frac{\hat{S}_A(t) - \hat{S}_B(t)}{\ln(\hat{S}_A(t)) - \ln(\hat{S}_B(t))} = \lim_{\hat{S}_B(t) \to \hat{S}_A(t)} \frac{-1}{-1/\hat{S}_B(t)} = \lim_{\hat{S}_B(t) \to \hat{S}_A(t)} \hat{S}_B(t) = \hat{S}_A(t) \;.
    \]
    This limit necessitates the piecewise definition $\hat{S}^*(t) = \hat{S}_A(t)$ at the crossing point. Substituting this continuous limit into the right side yields $-(0)/\hat{S}_A(t) = 0$. The equality holds trivially, and the mapping remains perfectly continuous.

    Thus, the proposition holds continuously for all $t \in \mathcal{T}$.
\end{proof}

\newpage

\section{Additional Results}
\label{apx:additional}

In this section, we provide additional experimental results on our synthetic benchmark, where we vary different aspects of the data and subgroup generation processes. Also, we provide additional results from our real-world data regarding subgroup sizes and rule lengths found by each method, along with the runtimes. Lastly, we provide the full rules learned in our case study along with their pruned versions.

\subsection*{Synthetic Data}

We assess the runtimes of \ourmethod and each of baselines from our scalability experiment in Fig.~\ref{fig:synth_panel_bands} (left). See Appx.~\ref{apx:complexity} for details regarding computational complexity. In Fig.~\ref{fig:additional_synth_panel} (left), we see that all the methods have consistent runtimes except \EDS, which varies quite a lot as the number of dimensions increases. Expectedly, \EDS and \RK require increasingly more time unlike \ourmethod and \FIBERS. We attribute some of the increase in runtime of \ourmethod to learning the population model, which varies with the choice of the regressor. Overall, \ourmethod is the fastest method when factoring in retrieval performance. \EDS is the closest competitor but needs around 4 times more time for 1000 features.

Next, we assess the retrieval performance of each method under varying population-subgroup hazard ratios in Fig.~\ref{fig:additional_synth_panel} (center). The hazard ratio indicates the rate by which the population outsurvives the subgroup, or vice versa. We can see that \ourmethod consistently outperforms all baselines across all hazard ratios. As the hazard ratio approaches 1.0, the retrieval performance of all methods degrades since the survival distributions of the subgroup and population become increasingly similar.

Lastly, we assess the retrieval performance of each method under increasingly large datasets in Fig.~\ref{fig:additional_synth_panel} (right). We see that \ourmethod consistently outperforms all baselines across all dataset sizes. As the dataset size increases, the stability of \ourmethod improves, since more data is available to learn from.

\begin{figure}
    \centering
    \scriptsize
    \input{texfigs/additional_synth_panel.tex}
    \caption{\emph{Synthetic setting.} Comparison of \ourmethod and each of \RK, \EDS, and \FIBERS in terms of runtime recovering planted subgroups with increasingly large dataset dimensionalities \textbf{(Left)}. Lower is better. Also, in terms of $F_1$-score, comparisons between \ourmethod and baselines under varying population-subgroup hazard ratios \textbf{(Center)}, where the hazard ratio indicates the rate by which the population outsurvives the subgroup, or vice versa, and increasingly large datasets \textbf{(Right)}. Higher is better. \EDS is the closest competitor to \ourmethod closely followed by \RK. The shaded areas show ±1 standard error over 10 runs.}
    \label{fig:additional_synth_panel}
\end{figure}

\subsection*{Real-World Data}

In Table \ref{tab:subgroup_sizes}, we present the average subgroup sizes discovered by each method across datasets in our real-world experiments. We can see that \ourmethod consistently discovers much smaller subgroups than all baselines. This is expected since \ourmethod directly optimizes for subgroup size via the size penalty $\gamma$. In contrast, the baselines do not have a direct mechanism for controlling subgroup size, leading to larger subgroups. For TRACE, the subgroup size for \ourmethod changed from 1878.0 to 470.0 after adjusting $\gamma$.

In Table \ref{tab:rule_lengths}, we see that \ourmethod and \FIBERS result in rules with a median length of $\geq 3$ conditions. In contrast, each of \RK and \EDS results in a median rule length of one. We highlight the rule length of \ourmethod on the TRACE dataset because since the gamma values we tried for all datasets were too high for this dataset, the rule length is 0.0. We try other gamma values and report the rule length after adjustment to be 3.0.

In Table \ref{tab:runtimes}, we see that \ourmethod completes between 9 seconds and a little under 9 minutes with a median runtime of 18.6 seconds. We see that \RK is the fastest with a median of 2 seconds, while \EDS is the second fastest with a median of 8 seconds but with some extremely slow runs. \ourmethod comes in third with one quite slow run (9 minutes). \FIBERS is the slowest with a median of about a minute, however, it consistently takes around a minute to complete for all datasets. Nevertheless, \ourmethod requires little/competitive time when applied to benchmark real-world datasets.

\begin{table*}[]
    \centering
    \scriptsize
    \caption{Average subgroup sizes (with coverage percentages) per method across real-world datasets.}
    \label{tab:subgroup_sizes}
    \begin{tabular}{lrrrrr}
        \toprule
                  & $n$     & \ourmethod & \RK    & \EDS    & \FIBERS  \\ 
        \midrule
        UnempDur         & 3241    & 112.0 (3.5\%)  & 1620.5 (50.0\%) & 1620.5 (50.0\%) & 1585.6 (48.9\%) \\
        Nwtco            & 4028    & 168.0 (4.2\%)  & 2014.0 (50.0\%) & 1490.0 (37.0\%) & 4027.0 (100.0\%)\\
        Rott2            & 2982    & 320.0 (10.7\%) & 1825.7 (61.2\%) & 1007.2 (33.8\%) & 1210.6 (40.6\%) \\
        Rdata            & 1040    & 140.0 (13.5\%) & 520.0 (50.0\%)  & 268.8 (25.8\%)  & 503.2 (48.4\%)  \\
        Aids2            & 2839    & 306.0 (10.8\%) & 1342.6 (47.3\%) & 503.7 (17.7\%)  & 2834.0 (99.8\%) \\
        Dialysis         & 6805    & 2665.0 (39.2\%)& 4095.7 (60.2\%) & 2636.8 (38.7\%) & 2568.6 (37.7\%) \\
        TRACE            & 1878    & 1878.0 (100.0\%)& 1113.5 (59.3\%) & 657.3 (35.0\%)  & 1877.0 (99.9\%) \\
        Support2         & 9105    & 33.0 (0.4\%)   & 265.0 (2.9\%)   & 204.8 (2.3\%)   & 451.0 (5.0\%)   \\
        DataDIVAT2       & 1837    & 63.0 (3.4\%)   & 918.5 (50.0\%)  & 720.7 (39.2\%)  & 870.2 (47.4\%)  \\
        ProstateSurv.    & 14294   & 1840.0 (12.9\%)& 7147.0 (50.0\%) & 5666.5 (39.6\%) & 11127.2 (77.8\%)\\
        Actg             & 1151    & 38.0 (3.3\%)   & 445.0 (38.7\%)  & 334.6 (29.1\%)  & 1149.0 (99.8\%) \\
        Scania           & 1931    & 122.0 (6.3\%)  & 729.6 (37.8\%)  & 431.2 (22.3\%)  & 1272.2 (65.9\%) \\
        Grace            & 1000    & 33.0 (3.3\%)   & 500.0 (50.0\%)  & 321.0 (32.1\%)  & 760.2 (76.0\%)  \\ 
        \midrule
        Q1 (25\%) Cov.   & -       & 3.40\%         & 47.30\%         & 25.80\%         & 47.40\%         \\
        Median (Q2) Cov. & -       & 6.30\%         & 50.00\%         & 33.80\%         & 65.90\%         \\
        Q3 (75\%) Cov.   & -       & 12.90\%        & 50.00\%         & 38.70\%         & 99.80\%         \\
        \bottomrule
    \end{tabular}
\end{table*}

\begin{table*}[]
    \parbox{0.49\textwidth}{
        \setlength{\tabcolsep}{5pt}
        \centering
        \tiny
        \caption{Average rule length (conditions) per method across real-world datasets with $p$ features.}
        \label{tab:rule_lengths}
        \begin{tabular}{lrrrrr}
            \toprule
                    & $p$ & \ourmethod  & \RK & \EDS & \FIBERS  \\ 
            \midrule
            UnempDur         & 7   & 3.0     & 1.0     & 1.0     & 1.2     \\
            Nwtco            & 14  & 5.0     & 1.0     & 1.2     & 7.0     \\
            Rott2            & 18  & 1.0     & 1.0     & 1.2     & 3.6     \\
            Rdata            & 8   & 2.0     & 1.0     & 1.33    & 5.0     \\
            Aids2            & 15  & 2.0     & 2.2     & 2.33    & 8.0     \\
            Dialysis         & 74  & 33.0    & 1.0     & 17.6    & 12.6    \\
            TRACE            & 10  & 0.0    & 1.25    & 1.0     & 11.4    \\
            Support2         & 76  & 10.0    & 1.0     & 2.0     & 1.8     \\
            DataDIVAT2       & 7   & 2.0     & 1.0     & 1.33    & 1.2     \\
            ProstateSurv.    & 9   & 5.0     & 1.0     & 1.5     & 1.8     \\
            Actg             & 24  & 5.0     & 3.2     & 1.6     & 5.2     \\
            Scania           & 12  & 3.0     & 1.6     & 1.4     & 6.0     \\
            Grace            & 7   & 4.0     & 2.5     & 1.2     & 8.4     \\
            \midrule
            Q1 (25\%)        & -   & 2.0     & 1.0     & 1.2     & 1.8     \\
            Median (Q2)      & -   & 3.0     & 1.0     & 1.33    & 5.2     \\
            Q3 (75\%)        & -   & 5.0     & 1.6     & 1.6     & 8.0     \\
            \bottomrule
        \end{tabular}
    }
    \hfill
    \parbox{0.49\textwidth}{
        \centering
        \tiny
        \caption{Runtime (seconds) per method across real-world datasets.}
        \label{tab:runtimes}
        \begin{tabular}{lrrrr}
            \toprule
                            & \ourmethod    & \RK   & \EDS    & \FIBERS   \\ 
            \midrule
            UnempDur         & 15.24     & 1.93      & 0.82       & 65.07    \\
            Nwtco            & 75.22     & 2.59      & 81.97      & 69.65    \\
            Rott2            & 20.52     & 3.20      & 7.82       & 70.07    \\
            Rdata            & 11.17     & 0.46      & 1.76       & 60.83    \\
            Aids2            & 18.64     & 2.55      & 758.92     & 61.93    \\
            Dialysis         & 525.30    & 65.11     & 4826.16    & 72.71    \\
            TRACE            & 17.32     & 1.08      & 2.30       & 62.28    \\
            Support2         & 34.42     & 3.86      & 20.84      & 65.12    \\
            DataDIVAT2       & 13.00     & 0.61      & 5.12       & 57.25    \\
            ProstateSurv.    & 28.12     & 12.93     & 7.18       & 71.34    \\
            Actg             & 14.35     & 1.47      & 21.44      & 60.96    \\
            Scania           & 19.75     & 1.16      & 51.05      & 63.21    \\
            Grace            & 9.09      & 0.39      & 1.99       & 61.49    \\
            \midrule
            Q1 (25\%)        & 14.35     & 1.08      & 2.30       & 61.49    \\
            Median (Q2)      & 18.64     & 1.93      & 7.82       & 63.21    \\
            Q3 (75\%)        & 28.12     & 3.20      & 51.05      & 69.65    \\
            \bottomrule
        \end{tabular}
    }
\end{table*}

\subsection*{Case Study}
We present in Tables~\ref{tab:prune} and~\ref{tab:prune2} examples of rules for the primary, respectively, postoperative RCT cohorts from our case study. We also present the corresponding pruned rules, which describe subgroups whose Jaccard similarities in terms of the membership indicators to the original subgroups are no less than the threshold of 0.95. We also provide the resulting changes in subgroup size and exceptionality. For this, we use the post hoc pruning algorithm \textsc{PruneRule}, which we introduce in Alg.~\ref{alg:prune}.

We can see that the rules are reduced to a combination of a handful of conditions after pruning. In the very first rule, we see a drastic reduction in rule size. This goes to show how the problem of collinearity is amplified when $n\approx p$, when in reality, rules having only a handful of conditions could suffice. Nevertheless, this post hoc pruning method cannot narrow down the predicates to those that, in the gene expression case, activate the rest of the genes in the same gene signature.

\begin{table*}[]
    \centering
    \tiny
    \caption{Rules learned on the primary RCT cohort before and after pruning using a Jaccard similarity threshold of $0.95$ for membership indicators along with the respective changes in subgroup sizes and exceptionality.}
    \label{tab:prune}
    \begin{tabular}{p{0.4\textwidth}|p{0.4\textwidth}|p{0.1\textwidth}}
        \toprule
        \textbf{Original rule} & \textbf{Pruned rule} & \textbf{Changes} \\
        \midrule
        \emph{“$\mathit{ACTN1} \in [-3.3, -0.9] \land \mathit{AKT1} \in [-3.6, -1.9] \land \mathit{ALDH3A1} \in [-8.9, -1.4] \land \mathit{ANLN} \in [-4.3, -2.2] \land \mathit{BCL2L1} \in [-7.2, -5.3] \land \mathit{BNIP3L} \in [-3.0, -0.6] \land \mathit{CAV1} \in [-4.4, -1.1] \land \mathit{CDKN3} \in [-4.5, -2.5] \land \mathit{CXCL12} \in [-8.2, -2.6] \land \mathit{DCBLD1} \in [-6.9, -4.7] \land \mathit{ERBB2} \in [-1.9, 0.7] \land \mathit{ERBB3} \in [-7.0, -2.9] \land \mathit{GNAI1} \in [-5.2, -0.0] \land \mathit{IGF1R} \in [-3.5, -0.9] \land \mathit{LOX} \in [-6.4, -3.3] \land \mathit{MTOR} \in [-3.2, -2.3] \land \mathit{NOTCH1} \in [-4.9, -2.6] \land \mathit{PLAU} \in [-4.5, -1.2] \land \mathit{RB1} \in [-3.8, -2.4] \land \mathit{RELA} \in [-3.2, -1.8] \land \mathit{RFC4} \in [-4.6, -1.6] \land \mathit{RPA2} \in [-4.0, -2.1] \land \mathit{SNAI1} \in [-6.9, -3.3] \land \mathit{TCF3} \in [-3.2, -1.3] \land \mathit{TGFB1} \in [-2.6, -1.0] \land \mathit{TPI1} \in [-7.2, -5.3] \land \mathit{XPA} \in [-4.7, -1.6] \land \mathit{XRCC1} \in [-4.2, -2.6]$”} & \emph{“$\mathit{CDKN3} \in [-4.5, -2.5] \land \mathit{ERBB2} \in [-1.9, 0.7] \land \mathit{LOX} \in [-6.4, -3.3] \land \mathit{RB1} \in [-3.8, -2.4] \land \mathit{RPA2} \in [-4.0, -2.1] \land \mathit{TGFB1} \in [-2.6, -1.0]$”} & Subgroup size: $6\to6$ \newline\newline Exceptionality: $19.1051\to19.1037$ \\
        \midrule
        \emph{“$\mathit{AKT1} \in [-3.2, -1.2] \land \mathit{ALDH1A1} \in [-8.1, -2.3] \land \mathit{ASS1} \in [-4.4, -0.4] \land \mathit{CENPK} \in [-6.1, -4.4] \land \mathit{CHEK2} \in [-5.5, -3.8] \land \mathit{CXCR4} \in [-6.6, -2.8] \land \mathit{CYP1B1} \in [-11.3, -4.5] \land \mathit{ERBB3} \in [-5.5, -3.5] \land \mathit{ERBB4} \in [-10.6, -7.6] \land \mathit{FGFR3} \in [-7.1, -3.1] \land \mathit{FOSL1} \in [-5.3, -0.7] \land \mathit{GNAI1} \in [-3.2, -0.4] \land \mathit{GPI} \in [-2.0, 0.2] \land \mathit{HIF1A} \in [-6.8, -4.7] \land \mathit{KDR} \in [-5.5, -3.5] \land \mathit{LDHA} \in [-6.9, -4.3] \land \mathit{LGALS1} \in [-2.6, 1.0] \land \mathit{LIMD1} \in [-6.2, -4.7] \land \mathit{LOX} \in [-4.1, -0.9] \land \mathit{MCM6} \in [-5.9, -3.7] \land \mathit{MET} \in [-5.3, -2.8] \land \mathit{MMP7} \in [-9.2, -3.4] \land \mathit{MYC} \in [-2.8, 0.3] \land \mathit{NBN} \in [-3.7, -2.7] \land \mathit{PDK1} \in [-4.9, -3.5] \land \mathit{PFKFB3} \in [-3.9, -0.7] \land \mathit{RASSF6} \in [-7.6, -5.0] \land \mathit{RB1} \in [-4.4, -2.5] \land \mathit{SYK} \in [-6.0, -3.9] \land \mathit{SYNGR3} \in [-8.9, -4.7] \land \mathit{TPI1} \in [-7.2, -4.8] \land \mathit{XRCC1} \in [-5.4, -3.4] \land \mathit{CD44} \in [-0.1, 2.1]$”} & \emph{“$\mathit{ALDH1A1} \in [-8.1, -2.3] \land \mathit{ASS1} \in [-4.4, -0.4] \land \mathit{CYP1B1} \in [-11.3, -4.5] \land \mathit{ERBB3} \in [-5.5, -3.5] \land \mathit{ERBB4} \in [-10.6, -7.6] \land \mathit{FOSL1} \in [-5.3, -0.7] \land \mathit{GNAI1} \in [-3.2, -0.4] \land \mathit{GPI} \in [-2.0, 0.2] \land \mathit{HIF1A} \in [-6.8, -4.7] \land \mathit{KDR} \in [-5.5, -3.5] \land \mathit{LDHA} \in [-6.9, -4.3] \land \mathit{LGALS1} \in [-2.6, 1.0] \land \mathit{LIMD1} \in [-6.2, -4.7] \land \mathit{LOX} \in [-4.1, -0.9] \land \mathit{MCM6} \in [-5.9, -3.7] \land \mathit{MET} \in [-5.3, -2.8] \land \mathit{MMP7} \in [-9.2, -3.4] \land \mathit{MYC} \in [-2.8, 0.3] \land \mathit{NBN} \in [-3.7, -2.7] \land \mathit{PDK1} \in [-4.9, -3.5] \land \mathit{PFKFB3} \in [-3.9, -0.7] \land \mathit{RASSF6} \in [-7.6, -5.0] \land \mathit{RB1} \in [-4.4, -2.5] \land \mathit{SYK} \in [-6.0, -3.9] \land \mathit{SYNGR3} \in [-8.9, -4.7] \land \mathit{TPI1} \in [-7.2, -4.8] \land \mathit{XRCC1} \in [-5.4, -3.4] \land \mathit{CD44} \in [-0.1, 2.1]$”} & Subgroup size: $12\to12$ \newline\newline Exceptionality: $16.6102\to16.6079$ \\
        \midrule
        \emph{“$\mathit{ADM} \in [-5.3, -0.8] \land \mathit{AKT1} \in [-3.6, -1.7] \land \mathit{ALDH3A1} \in [-8.9, -0.5] \land \mathit{ANLN} \in [-4.3, -2.8] \land \mathit{ASS1} \in [-4.9, -0.0] \land \mathit{ATP5G3} \in [-1.3, 1.4] \land \mathit{ATR} \in [-5.9, -3.8] \land \mathit{BCL2L1} \in [-7.8, -4.9] \land \mathit{BIRC5} \in [-4.7, -1.0] \land \mathit{CAV1} \in [-4.4, 0.3] \land \mathit{CDKN3} \in [-5.9, -2.5] \land \mathit{CLDN4} \in [-5.0, -0.6] \land \mathit{CXCL12} \in [-8.2, -3.2] \land \mathit{ERBB3} \in [-6.3, -3.2] \land \mathit{ERCC5} \in [-5.2, -3.6] \land \mathit{FAM83B} \in [-7.4, -1.7] \land \mathit{FLT1} \in [-5.7, -3.1] \land \mathit{GNAI1} \in [-5.2, -0.2] \land \mathit{IGF1R} \in [-3.5, -0.6] \land \mathit{INHBA} \in [-4.4, 1.9] \land \mathit{KDR} \in [-5.5, -2.8] \land \mathit{LDHA} \in [-7.5, -5.7] \land \mathit{LGALS1} \in [-2.1, 1.0] \land \mathit{LOX} \in [-6.4, -1.4] \land \mathit{MDM2} \in [-4.6, -2.2] \land \mathit{MMP13} \in [-9.3, 1.3] \land \mathit{MRE11A} \in [-6.1, -3.6] \land \mathit{MTOR} \in [-3.7, -2.3] \land \mathit{NOTCH1} \in [-5.9, -2.6] \land \mathit{P4HA2} \in [-4.7, -1.4] \land \mathit{PSMD9} \in [-4.2, -2.1] \land \mathit{PTEN} \in [-7.3, -5.6] \land \mathit{RAD23B} \in [-0.9, 0.8] \land \mathit{RB1} \in [-3.8, -2.3] \land \mathit{RPA2} \in [-4.9, -2.1] \land \mathit{SERPINB2} \in [-6.6, -1.2] \land \mathit{TCF3} \in [-3.5, -1.0] \land \mathit{TGFB1} \in [-2.6, -0.3] \land \mathit{TPI1} \in [-7.2, -5.0] \land \mathit{XRCC1} \in [-5.3, -2.6] \land \mathit{XRCC5} \in [-2.5, -0.8] \land \mathit{CD44} \in [-0.8, 2.5]$”} & \emph{“$\mathit{ANLN} \in [-4.3, -2.8] \land \mathit{ATR} \in [-5.9, -3.8] \land \mathit{CLDN4} \in [-5.0, -0.6] \land \mathit{CXCL12} \in [-8.2, -3.2] \land \mathit{ERBB3} \in [-6.3, -3.2] \land \mathit{ERCC5} \in [-5.2, -3.6] \land \mathit{LDHA} \in [-7.5, -5.7] \land \mathit{LGALS1} \in [-2.1, 1.0] \land \mathit{MDM2} \in [-4.6, -2.2] \land \mathit{MTOR} \in [-3.7, -2.3] \land \mathit{PTEN} \in [-7.3, -5.6] \land \mathit{RAD23B} \in [-0.9, 0.8] \land \mathit{RB1} \in [-3.8, -2.3] \land \mathit{SERPINB2} \in [-6.6, -1.2] \land \mathit{TCF3} \in [-3.5, -1.0] \land \mathit{TGFB1} \in [-2.6, -0.3] \land \mathit{TPI1} \in [-7.2, -5.0]$”} & Subgroup size: $28\to29$ \newline\newline Exceptionality: $16.0098\to15.7934$ \\
        \midrule
        \emph{“$\mathit{AKT1} \in [-3.6, -1.5] \land \mathit{ALDH3A1} \in [-8.9, -0.5] \land \mathit{ANLN} \in [-4.5, -2.6] \land \mathit{ASS1} \in [-4.9, -0.0] \land \mathit{ATP5G3} \in [-1.3, 1.7] \land \mathit{ATR} \in [-5.9, -3.7] \land \mathit{BCL2L1} \in [-7.7, -5.0] \land \mathit{BIRC5} \in [-5.3, -1.3] \land \mathit{BNIP3L} \in [-3.1, -0.6] \land \mathit{CAV1} \in [-4.4, 0.4] \land \mathit{CDKN3} \in [-5.8, -2.5] \land \mathit{CLDN4} \in [-5.0, -0.7] \land \mathit{CXCL12} \in [-8.2, -2.2] \land \mathit{ERBB3} \in [-6.2, -3.2] \land \mathit{ERBB4} \in [-10.6, -7.4] \land \mathit{ERCC5} \in [-5.2, -3.6] \land \mathit{FAM83B} \in [-6.8, -1.6] \land \mathit{GNAI1} \in [-5.2, -0.2] \land \mathit{HIF1A} \in [-7.7, -4.7] \land \mathit{IGF1R} \in [-3.5, -0.6] \land \mathit{INHBA} \in [-4.8, 1.8] \land \mathit{KDR} \in [-5.5, -3.6] \land \mathit{LDHA} \in [-7.5, -5.3] \land \mathit{LGALS1} \in [-2.6, 1.0] \land \mathit{LOX} \in [-6.4, -1.7] \land \mathit{MCM6} \in [-5.9, -3.0] \land \mathit{MMP13} \in [-9.3, 1.3] \land \mathit{MRGBP} \in [-5.0, -2.8] \land \mathit{MTOR} \in [-3.8, -2.3] \land \mathit{P4HA2} \in [-4.7, -1.4] \land \mathit{PDK1} \in [-4.9, -2.3] \land \mathit{PGAM1} \in [-2.6, -0.7] \land \mathit{PSMD9} \in [-4.2, -2.0] \land \mathit{PTEN} \in [-7.3, -5.5] \land \mathit{RAD23B} \in [-0.8, 0.8] \land \mathit{RB1} \in [-3.9, -2.3] \land \mathit{TCF3} \in [-3.4, -1.4] \land \mathit{TGFB1} \in [-2.6, -0.0] \land \mathit{TP53} \in [-4.3, -1.0] \land \mathit{TPI1} \in [-7.2, -4.9] \land \mathit{CD44} \in [-0.4, 2.1]$”} & \emph{“$\mathit{ALDH3A1} \in [-8.9, -0.5] \land \mathit{CLDN4} \in [-5.0, -0.7] \land \mathit{ERBB3} \in [-6.2, -3.2] \land \mathit{ERBB4} \in [-10.6, -7.4] \land \mathit{HIF1A} \in [-7.7, -4.7] \land \mathit{LDHA} \in [-7.5, -5.3] \land \mathit{MMP13} \in [-9.3, 1.3] \land \mathit{MRGBP} \in [-5.0, -2.8] \land \mathit{P4HA2} \in [-4.7, -1.4] \land \mathit{PDK1} \in [-4.9, -2.3] \land \mathit{PGAM1} \in [-2.6, -0.7] \land \mathit{PTEN} \in [-7.3, -5.5] \land \mathit{RAD23B} \in [-0.8, 0.8] \land \mathit{TP53} \in [-4.3, -1.0] \land \mathit{TPI1} \in [-7.2, -4.9]$”} & Subgroup size: $25\to26$ \newline\newline Exceptionality: $15.8073\to15.2069$ \\
        \bottomrule
    \end{tabular}
\end{table*}

\clearpage

\begin{table*}[]
    \centering
    \tiny
    \caption{Rules learned on the postoperative RCT cohort before and after pruning using a Jaccard similarity threshold of $0.95$ for membership indicators along with the respective changes in subgroup sizes and exceptionality.}
    \label{tab:prune2}
    \begin{tabular}{p{0.4\textwidth}|p{0.4\textwidth}|p{0.1\textwidth}}
        \toprule
        \textbf{Original rule} & \textbf{Pruned rule} & \textbf{Changes} \\
        \midrule
        \emph{“$\mathit{ALDH1A1} \in [-7.0, -3.2] \land \mathit{BCL2L1} \in [-6.9, -5.5] \land \mathit{BIRC5} \in [-4.6, -1.6] \land \mathit{BNIP3L} \in [-2.5, -1.1] \land \mathit{CA9} \in [-9.0, -4.4] \land \mathit{CBX4} \in [-5.3, -3.7] \land \mathit{CLDN4} \in [-5.3, -1.3] \land \mathit{CXCL12} \in [-5.2, -1.1] \land \mathit{EGLN3} \in [-4.4, -0.9] \land \mathit{ENO1} \in [-0.2, 1.2] \land \mathit{ENO2} \in [-5.8, -2.4] \land \mathit{EPHA1} \in [-7.2, -4.6] \land \mathit{GPI} \in [-1.7, -0.2] \land \mathit{HSPB1} \in [-4.0, -2.0] \land \mathit{ITGB2} \in [-4.7, -2.3] \land \mathit{KCTD11} \in [-2.3, 0.1] \land \mathit{MAP2K2} \in [-4.3, -3.3] \land \mathit{MCM6} \in [-5.8, -4.0] \land \mathit{MMP13} \in [-1.9, 1.7] \land \mathit{MMP9} \in [-2.7, 0.6] \land \mathit{MPRS17} \in [-5.0, -1.2] \land \mathit{PFKFB3} \in [-2.7, -0.2] \land \mathit{PTEN} \in [-6.6, -5.1] \land \mathit{RAD23B} \in [-1.0, 0.9] \land \mathit{RB1} \in [-5.0, -2.5] \land \mathit{RELA} \in [-2.9, -1.6] \land \mathit{SERPINB2} \in [-7.6, -2.6] \land \mathit{SLC5A1} \in [-9.9, -3.6] \land \mathit{TCF3} \in [-3.8, -2.3] \land \mathit{XRCC1} \in [-5.2, -3.2] \land \mathit{XRCC4} \in [-6.1, -4.2] \land \mathit{XRCC5} \in [-2.1, -0.6]$”} & \emph{“$\mathit{EPHA1} \in [-7.2, -4.6] \land \mathit{GPI} \in [-1.7, -0.2] \land \mathit{KCTD11} \in [-2.3, 0.1] \land \mathit{MAP2K2} \in [-4.3, -3.3] \land \mathit{MMP13} \in [-1.9, 1.7] \land \mathit{MMP9} \in [-2.7, 0.6] \land \mathit{XRCC5} \in [-2.1, -0.6]$”} & Subgroup size: $12\to12$ \newline\newline Exceptionality: $27.8189\to27.7435$ \\
        \midrule
        \emph{“$\mathit{ANXA5} \in [-3.0, -0.9] \land \mathit{ATM} \in [-3.8, -1.6] \land \mathit{ATP5G3} \in [-1.1, 1.0] \land \mathit{BNIP3} \in [-6.1, -2.6] \land \mathit{BSG} \in [-9.4, -5.4] \land \mathit{CD24} \in [-8.8, -4.5] \land \mathit{CXCL12} \in [-6.2, -1.3] \land \mathit{DCBLD1} \in [-7.5, -4.5] \land \mathit{DKK3} \in [-5.3, -1.4] \land \mathit{FANCA} \in [-6.7, -3.1] \land \mathit{FN1} \in [-4.6, 1.8] \land \mathit{HIF1A} \in [-7.1, -4.6] \land \mathit{HK2} \in [-6.7, -2.6] \land \mathit{HSPA4} \in [-3.9, -2.2] \land \mathit{INHBA} \in [-7.9, -1.3] \land \mathit{ITGB1} \in [-2.5, -0.1] \land \mathit{KRT17} \in [-3.2, 2.8] \land \mathit{LGALS1} \in [-3.4, -0.2] \land \mathit{LOXL2} \in [-6.4, -3.0] \land \mathit{MCM6} \in [-5.3, -2.5] \land \mathit{MDM2} \in [-4.3, -0.7] \land \mathit{MMP10} \in [-9.7, -1.8] \land \mathit{MMP13} \in [-10.4, -2.8] \land \mathit{MMP2} \in [-5.0, -0.5] \land \mathit{MTOR} \in [-3.9, -2.3] \land \mathit{MYC} \in [-3.7, -0.6] \land \mathit{MYNN} \in [-5.1, -3.0] \land \mathit{NR1D2} \in [-4.9, -2.3] \land \mathit{P4HA2} \in [-5.1, -2.4] \land \mathit{PSMD9} \in [-4.3, -2.2] \land \mathit{RMI2} \in [-5.1, -2.9] \land \mathit{RPA2} \in [-4.5, -2.1] \land \mathit{SFN} \in [-1.6, 3.1] \land \mathit{SLC3A2} \in [-4.3, -2.3] \land \mathit{SMDT1} \in [-3.9, -2.2] \land \mathit{SNAI1} \in [-7.6, -3.8] \land \mathit{SPP1} \in [-9.0, -3.9] \land \mathit{TCF3} \in [-4.0, -2.1] \land \mathit{TP53} \in [-3.8, -1.0] \land \mathit{XRCC4} \in [-5.2, -3.5] \land \mathit{XRCC5} \in [-2.8, -0.9] \land \mathit{YAP1} \in [-3.6, -0.2]$”} & \emph{“$\mathit{ATM} \in [-3.8, -1.6] \land \mathit{HIF1A} \in [-7.1, -4.6] \land \mathit{KRT17} \in [-3.2, 2.8] \land \mathit{MMP13} \in [-10.4, -2.8] \land \mathit{MMP2} \in [-5.0, -0.5] \land \mathit{PSMD9} \in [-4.3, -2.2] \land \mathit{RMI2} \in [-5.1, -2.9] \land \mathit{SNAI1} \in [-7.6, -3.8] \land \mathit{XRCC4} \in [-5.2, -3.5] \land \mathit{YAP1} \in [-3.6, -0.2]$”} & Subgroup size: $47\to45$ \newline\newline Exceptionality: $12.0529\to11.9264$ \\
        \midrule
        \emph{“$\mathit{ALDH1A1} \in [-7.0, -3.4] \land \mathit{BSG} \in [-8.9, -3.9] \land \mathit{CA9} \in [-9.5, -4.5] \land \mathit{CBX4} \in [-5.3, -3.7] \land \mathit{CLDN4} \in [-4.9, -0.4] \land \mathit{CXCL12} \in [-4.3, -1.5] \land \mathit{ENO1} \in [-0.3, 1.2] \land \mathit{ENO2} \in [-5.8, -2.4] \land \mathit{EPHA1} \in [-7.2, -4.6] \land \mathit{EPOR} \in [-10.2, -6.6] \land \mathit{ERCC5} \in [-5.2, -3.0] \land \mathit{HIF1A} \in [-5.7, -3.7] \land \mathit{LOX} \in [-2.9, -1.0] \land \mathit{MMP9} \in [-2.8, 0.6] \land \mathit{MYNN} \in [-5.2, -2.8] \land \mathit{PGK1} \in [-2.0, -0.2] \land \mathit{PTEN} \in [-6.5, -5.1] \land \mathit{RAD23B} \in [-1.0, 1.1] \land \mathit{RAD50} \in [-3.6, -1.6] \land \mathit{RB1} \in [-5.0, -2.4] \land \mathit{SLC5A1} \in [-10.0, -3.3] \land \mathit{TCF3} \in [-4.0, -2.1] \land \mathit{XRCC1} \in [-5.2, -4.1] \land \mathit{XRCC4} \in [-6.2, -4.1] \land \mathit{XRCC5} \in [-2.2, -1.0]$”} & \emph{“$\mathit{ALDH1A1} \in [-7.0, -3.4] \land \mathit{CA9} \in [-9.5, -4.5] \land \mathit{CBX4} \in [-5.3, -3.7] \land \mathit{ENO1} \in [-0.3, 1.2] \land \mathit{EPHA1} \in [-7.2, -4.6] \land \mathit{LOX} \in [-2.9, -1.0] \land \mathit{MMP9} \in [-2.8, 0.6] \land \mathit{MYNN} \in [-5.2, -2.8] \land \mathit{SLC5A1} \in [-10.0, -3.3] \land \mathit{TCF3} \in [-4.0, -2.1] \land \mathit{XRCC1} \in [-5.2, -4.1]$”} & Subgroup size: $12\to12$ \newline\newline Exceptionality: $27.4009\to26.0718$ \\
        \midrule
        \emph{“$\mathit{ATM} \in [-4.7, -1.6] \land \mathit{BIRC5} \in [-4.2, -1.0] \land \mathit{BNIP3L} \in [-3.2, -1.2] \land \mathit{BSG} \in [-9.2, -3.9] \land \mathit{CA9} \in [-10.7, -2.9] \land \mathit{CBX4} \in [-5.3, -3.1] \land \mathit{CD24} \in [-8.3, -4.5] \land \mathit{CDKN3} \in [-6.1, -2.7] \land \mathit{CXCL12} \in [-7.5, -2.0] \land \mathit{DCBLD1} \in [-7.5, -4.5] \land \mathit{ENO2} \in [-6.2, -2.4] \land \mathit{ERCC5} \in [-5.7, -3.0] \land \mathit{HK2} \in [-6.7, -2.2] \land \mathit{HSPA4} \in [-4.0, -2.2] \land \mathit{KDR} \in [-6.0, -2.4] \land \mathit{KRT17} \in [-3.0, 5.1] \land \mathit{LOX} \in [-7.1, -1.5] \land \mathit{LOXL2} \in [-6.4, -1.8] \land \mathit{MCM6} \in [-4.8, -2.5] \land \mathit{MME} \in [-9.6, -4.0] \land \mathit{MRE11A} \in [-6.1, -4.0] \land \mathit{MTOR} \in [-4.2, -2.3] \land \mathit{MYC} \in [-3.7, -0.5] \land \mathit{MYNN} \in [-5.0, -2.3] \land \mathit{NBN} \in [-3.9, -1.6] \land \mathit{NOTCH1} \in [-6.3, -1.9] \land \mathit{RAD23B} \in [-1.4, 1.1] \land \mathit{RAD50} \in [-4.0, -1.6] \land \mathit{RASSF6} \in [-8.1, -1.6] \land \mathit{RB1} \in [-5.0, -2.3] \land \mathit{RELA} \in [-3.5, -1.4] \land \mathit{SDHA} \in [-6.2, -2.9] \land \mathit{SERPINB2} \in [-8.0, -0.6] \land \mathit{TCF3} \in [-3.9, -2.1] \land \mathit{TP53} \in [-5.0, -1.0] \land \mathit{XPA} \in [-5.3, -2.6] \land \mathit{XPC} \in [-5.7, -3.7] \land \mathit{XRCC1} \in [-5.2, -2.5] \land \mathit{XRCC4} \in [-5.8, -3.5] \land \mathit{XRCC5} \in [-2.8, -0.8]$”} & \emph{“$\mathit{BIRC5} \in [-4.2, -1.0] \land \mathit{BSG} \in [-9.2, -3.9] \land \mathit{CA9} \in [-10.7, -2.9] \land \mathit{CD24} \in [-8.3, -4.5] \land \mathit{DCBLD1} \in [-7.5, -4.5] \land \mathit{LOX} \in [-7.1, -1.5] \land \mathit{LOXL2} \in [-6.4, -1.8] \land \mathit{MCM6} \in [-4.8, -2.5] \land \mathit{MME} \in [-9.6, -4.0] \land \mathit{MTOR} \in [-4.2, -2.3] \land \mathit{MYC} \in [-3.7, -0.5] \land \mathit{MYNN} \in [-5.0, -2.3] \land \mathit{RB1} \in [-5.0, -2.3] \land \mathit{TCF3} \in [-3.9, -2.1] \land \mathit{XPC} \in [-5.7, -3.7] \land \mathit{XRCC1} \in [-5.2, -2.5]$”} & Subgroup size: $85\to87$ \newline\newline Exceptionality: $11.2782\to11.1260$ \\
        \bottomrule
    \end{tabular}
\end{table*}

\clearpage

%% file: texfigs/additional_synth_panel.tex
\begin{tikzpicture}
    \scriptsize
    \begin{groupplot}[
        group style={
                group size=4 by 1,
                horizontal sep=55pt,
            },
            width=\textwidth/3.85,
            ylabel style={yshift=5pt},
            xlabel style={yshift=-5pt},
            ylabel={$F_1$},
            pretty grid line,
            ytick={0.2,0.4,0.6,0.8},
            extra y ticks={0,0,1,0},
            extra y tick style={grid=none},
            ymin=0,
            ymax=1,
            cycle list name=pr-colors-conf,
        ]

        \nextgroupplot[
            pretty line,
            pretty fill legend,
            pretty grid line,
            extra y ticks={0,18000},
            extra y tick style={grid=none},
            extra y tick labels={0,5},
            ylabel={Runtime (h)},
            xlabel={Number of features $p$},
            xmode=log,
            ymin=0,
            width=\textwidth/3.85,
            ymax=18000,
            log basis x=10,
            scaled x ticks=base 10:1,
            ytick={3600, 7200, 10800, 14400},
            yticklabels={1, 2, 3, 4},
            ylabel style={yshift=28pt},
            xlabel style={xshift=-28pt},
            cycle list name=pr-colors-conf,
        ]
            \foreach \i in {sysurv, rulekit, esmamds, fibers} {
                \addplot+[very thick] table[x=x, y=y] {data/exps/1kdims_runtime_\i.tsv};
                \addplot+[forget plot, draw=none, name path=low] table[x=x, y=y_c0] {data/exps/1kdims_runtime_\i.tsv};
                \addplot+[forget plot, draw=none, name path=up] table[x=x, y=y_c1] {data/exps/1kdims_runtime_\i.tsv};
                \addplot fill between[of=low and up];
            }

        \nextgroupplot[
            pretty line,
            pretty fill legend,
            xtick={1,5,9},
            xticklabels={0.2,1.0,1.8},
            xlabel={Population-Subgroup hazard ratio},
            cycle list name=pr-colors-conf,
        ]
            \foreach \i in {sysurv, rulekit, esmamds, fibers} {
                \addplot+[very thick] table[x=x, y=y] {data/exps/scale_f1_\i.tsv};
                \addplot+[forget plot, draw=none, name path=low] table[x=x, y=y_c0] {data/exps/scale_f1_\i.tsv};
                \addplot+[forget plot, draw=none, name path=up] table[x=x, y=y_c1] {data/exps/scale_f1_\i.tsv};
                \addplot fill between[of=low and up];
            }

        \nextgroupplot[
            pretty line,
            pretty fill legend,
            xlabel={Number of subjects $n$},
            xmode=log,
            log basis x=2,
            ylabel style={yshift=118pt},
            xlabel style={xshift=-118pt},
            cycle list name=pr-colors-conf,
            legend style={
                at={(1.4,0.4)},
                anchor=west,
                legend columns=1,
            },
            legend entries={\ourmethod,,\RK,,\EDS,,\FIBERS},
        ]
            \foreach \i in {sysurv, rulekit, esmamds, fibers} {
                \addplot+[very thick] table[x=x, y=y] {data/exps/dataset_size_f1_\i.tsv};
                \addplot+[forget plot, draw=none, name path=low] table[x=x, y=y_c0] {data/exps/dataset_size_f1_\i.tsv};
                \addplot+[forget plot, draw=none, name path=up] table[x=x, y=y_c1] {data/exps/dataset_size_f1_\i.tsv};
                \addplot fill between[of=low and up];
            }
    \end{groupplot}
\end{tikzpicture}